\documentclass{article}
\usepackage{iclr2027_conference,times}

\usepackage{amsmath,amsfonts,bm}

\def\eqref#1{equation~\ref{#1}}

\def\1{\bm{1}}

\DeclareMathAlphabet{\mathsfit}{\encodingdefault}{\sfdefault}{m}{sl}
\SetMathAlphabet{\mathsfit}{bold}{\encodingdefault}{\sfdefault}{bx}{n}

\usepackage{amsmath,amssymb,booktabs,multirow,graphicx,url}
\usepackage{colortbl,tabularx}
\usepackage{algorithm}
\usepackage{algpseudocode}
\usepackage{threeparttable}
\usepackage{graphicx}
\usepackage{wrapfig}
\usepackage{hyperref}

\newcommand{\method}{\textsc{Loopy}}

\newcommand{\pagelog}[1]{\typeout{LOOPY-SECTION-PAGE: #1 = \thepage}}
\title{\textsc{Loopy}: Low-Bit Quantization Framework for Looped Language Models}
\iclrfinalcopy
\author{%
Zeyu LI$^{1}$ \quad Yipu ZHANG$^{1}$ \quad Jintao Cheng$^{1}$ \quad Xin LI$^{2}$ \quad Wei ZHANG$^{1}$\\
\normalfont $^{1}$The Hong Kong University of Science and Technology\\
\normalfont $^{2}$Duke Kunshan University
}
\newtheorem{proposition}{Proposition} 
\begin{document}
\maketitle
\lhead{Preprint}
\suppressfloats[t]
\begin{abstract}
Looped language models provide a parameter-efficient way to scale iterative test-time computation by repeatedly executing a shared recurrent core. Post-training quantization (PTQ) can reduce the memory footprint and inference cost of looped language models, but errors introduced by a quantized shared core affect subsequent cores. Among PTQ methods, channel scaling and orthogonal rotations preserve the floating-point computation while producing representations with different quantization quality. We find that quantization configuration candidate rankings can change with recurrent depth, motivating configuration selection at the target deployment depth. However, evaluating every candidate over the full calibration set at this depth is costly. We therefore propose \method{}, a PTQ framework that formulates shared-core quantization through a recurrent-depth-aware objective, selecting shared low-bit representations by their final prediction loss at the target deployment depth. Channel scaling and orthogonal rotations parameterize the candidate representations. To approximately solve this selection problem efficiently, \method{} progressively allocates calibration windows to promising candidates while preserving complete target-depth execution, using only forward evaluations. Across eight settings, \method{} achieves the state-of-the-art results among different baselines. On Ouro-1.4B under W4A4, \method{} reduces LAMBADA perplexity by 36.5\% relative to SpinQuant. Our code is available at \url{https://github.com/Shameless0817/Loopy-review.git}.

\label{abstract-end}
\end{abstract}

\section{Introduction}
\label{sec:introduction}

Looped language models increase effective depth by repeatedly executing a shared recurrent core composed of transformer blocks, enabling iterative computation without storing a separate set of weights for every recurrent step as shown in Figure \ref{fig:shared-core} \citep{dehghani2019universal,zhu2025scaling,saunshi2025reasoning,geiping2026scaling}. Alongside token pruning \citep{cheng2026trainingfree}, PTQ enables efficient deployment through low-bit weights and activations \citep{xiao2023smoothquant,ashkboos2024quarot,liu2025spinquant}. Applying PTQ to looped language models therefore offers an effective solution to low-bit computation while preserving prediction quality.

\begin{wrapfigure}{r}{0.45\textwidth}
    \centering
    \includegraphics[width=\linewidth]{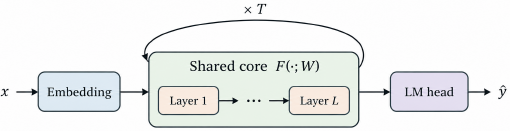}
    \caption{Looped language model.}
    \label{fig:shared-core}
\end{wrapfigure}


Channel scaling and orthogonal rotations mitigate quantization outliers while preserving the floating-point computation \citep{xiao2023smoothquant,lin2024awq,ashkboos2024quarot,liu2025spinquant}. However, these transformations do not uniquely determine a good quantization configuration \citep{liu2025spinquant}.


To preserve parameter sharing, we use the same quantized weights and transformations across recurrent steps (Section~\ref{sec:3_1}). As hidden states evolve across recurrent steps (Figure~\ref{fig:activation-steps}), quantization errors propagate through subsequent computation. Consequently, smaller local reconstruction error need not imply lower final prediction loss (Appendix~\ref{app:recurrent-error}). LoopQ also considers recurrent dynamics through loop-aware quantization~\citep{fang2026loopq}. These observations motivate target-depth evaluation, supported by our calibration-depth ablations (Figure~\ref{fig:calibration-depth}). However, evaluating every candidate on all calibration windows at this depth is costly.



Motivated by these observations, we introduce \method{}, a recurrent-depth-aware PTQ framework for looped language models, as shown in Figure~\ref{fig:loopy-overview}. The key idea is to formulate shared-core quantization as representation selection under a target-depth prediction objective: a shared quantized representation is evaluated by its final prediction loss after complete execution at the intended recurrent depth. Channel scaling and orthogonal rotations define the candidate representation space. \method{} approximately solves the resulting selection problem through progressive full-depth evaluation, keeping the recurrent depth fixed while allocating additional calibration windows to surviving candidates. This procedure uses only forward evaluations, and the selected quantized weights and transformations remain shared across recurrent steps.

\begin{figure}[t]
\centering
\includegraphics[width=\linewidth]{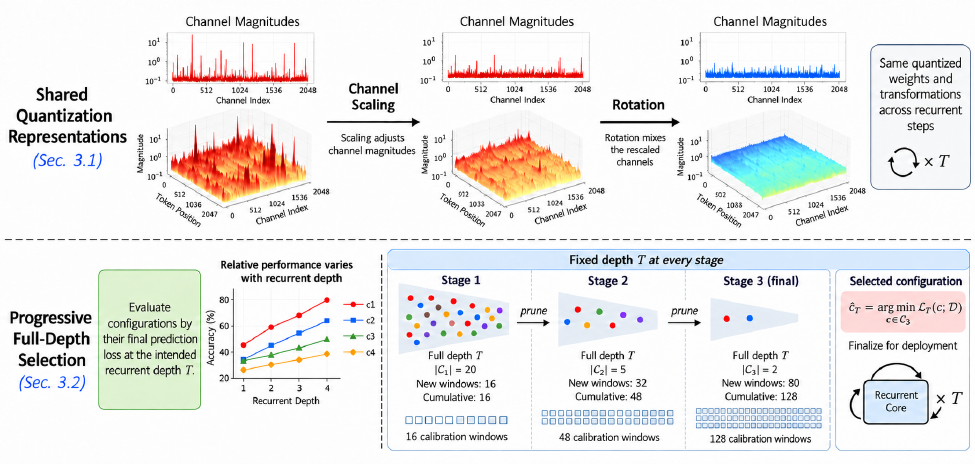}
\caption{\textbf{Loopy method overview.}
Top: channel scaling and rotation define quantized representations shared across recurrent steps.
Bottom: progressive selection allocates more calibration windows to surviving candidates while preserving the target depth $T$.}
\label{fig:loopy-overview}
\end{figure}

We evaluate \method{} on three representative looped language models: the Ouro series \citep{zhu2025scaling}, Huginn \citep{geiping2026scaling}, and Recurrent-Llama \citep{mcleish2025teaching}, under both W4A8 and W4A4 quantization. Our main contributions are summarized as follows.

\begin{itemize}
\item We systematically study low-bit quantization in looped language models. We uncover a loop-specific quantization phenomenon: a quantization configuration shared across recurrent steps can exhibit depth-dependent quality because of the recurrent dependency of looped models. This serves as the motivation for our method \method{}.

\item We introduce \method{}, a post-training quantization framework centered on a recurrent-depth-aware objective for selecting shared low-bit representations. \method{} parameterizes these representations through channel scaling and orthogonal rotations, and approximately solves the selection problem through progressive full-depth evaluation. This procedure preserves the target recurrent depth while reducing calibration work by allocating additional windows only to surviving candidates.

\item We evaluate \method{} across multiple looped language model families under W4A4 and W4A8 quantization. On Ouro-1.4B under W4A4, \method{} reduces LAMBADA perplexity by 36.5\% relative to the adapted SpinQuant baseline. We further evaluate \method{} with low-bit KV quantization and show that Q/K Hadamard transforms substantially mitigate the degradation from 4-bit KV quantization.
\end{itemize}

\section{Motivation}
\setcounter{topnumber}{1}
\pagelog{Motivation}
\label{sec:prelim}
 In a looped language model, a shared recurrent core repeatedly updates the hidden state according to
\[
h_{t+1}=F(h_t;W),\qquad t=0,\ldots,T-1,
\]
where $W$ denotes the shared weights and $T$ is the recurrent depth. After quantization, the corresponding computation becomes $\hat h_{t+1}=\hat F_c(\hat h_t)$, where $c$ specifies the quantized weights and transformation settings of the core's linear operators. The configuration $c$ is shared across recurrent steps, while hidden states and dynamic activation scales may vary. Appendix~\ref{app:activation-steps} illustrates activation changes across recurrent calls.

\subsection{Rotation Mitigates Activation Outliers}
We examine how rotation affects activation quantization on Ouro-1.4B. Figure~\ref{fig:rotation-benefit}(a) shows reduced activation peaks in the displayed block's attention and FFN inputs. We measure A4 reconstruction error using normalized mean-squared error (NMSE), $\|Q(a)-a\|_2^2/\|a\|_2^2$, where $Q$ denotes activation quantization. Figure~\ref{fig:rotation-benefit}(b) shows lower median NMSE after rotation for both inputs, motivating the shared rotation transforms in Section~\ref{sec:3_1}.

\begin{figure}[t]
\centering
\begin{minipage}[b]{0.73\linewidth}
\centering
\includegraphics[width=\linewidth]{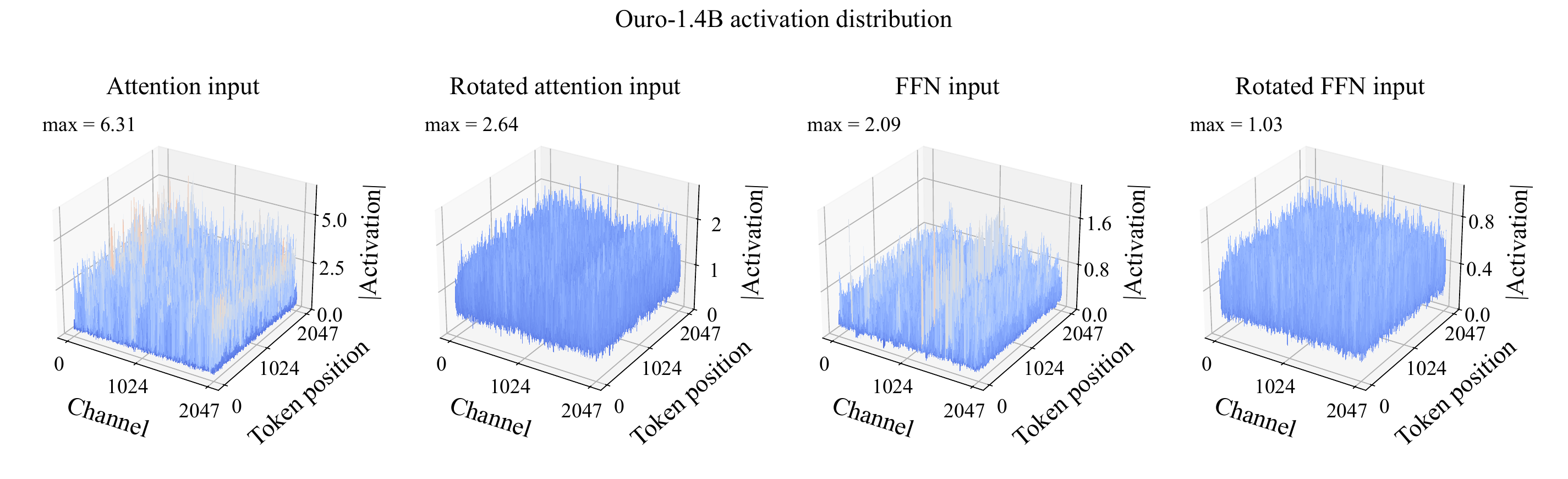}
\par{\small (a) Activation magnitudes.\par}
\end{minipage}\hfill
\begin{minipage}[b]{0.25\linewidth}
\centering
\includegraphics[width=\linewidth]{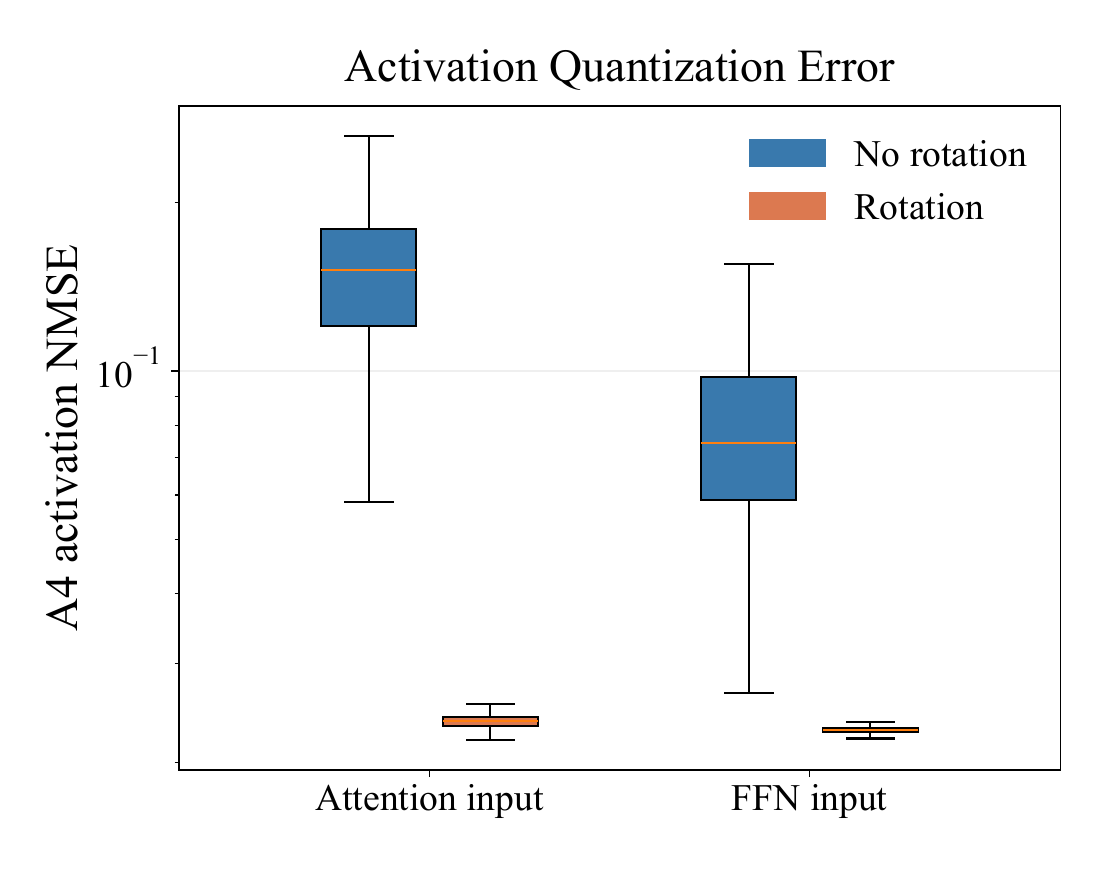}
\par{\small (b) A4 NMSE.\par}
\end{minipage}
\caption{\textbf{Rotation improves activation-side quantization on Ouro-1.4B.} (a) Attention and feed-forward network (FFN) inputs of physical block 22 at recurrent step 4, before and after rotation. Appendix~\ref{app:activation-distributions} provides all 24 blocks. (b) Activation normalized mean-squared error (NMSE) distributions without and with rotation, on a logarithmic scale.}
\label{fig:rotation-benefit}
\end{figure}

\subsection{Configuration Quality Depends on Recurrent Depth}
To examine whether configuration quality depends on recurrent depth, we evaluate 20 rotation--scaling configurations on Ouro-1.4B W4A4 at depths $T=1,\ldots,4$. Figure~\ref{fig:accuracy-depth} shows changes in their relative WinoGrande performance across depths, with the corresponding rankings provided in Appendix~\ref{app:configuration-rankings}.

\begin{figure}[t]
\centering
\includegraphics[width=\linewidth]{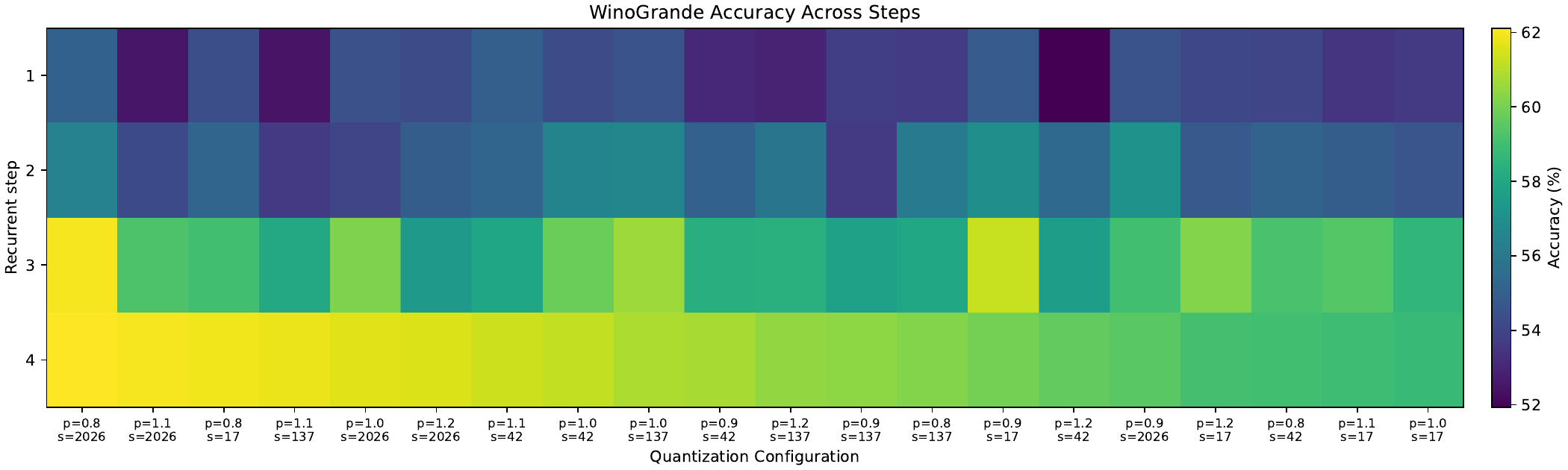}
\caption{\textbf{Accuracy across recurrent depths.} WinoGrande accuracy (\%) for 20 configurations of Ouro-1.4B W4A4. Columns identify scaling power $p$ and rotation seed $s$ and are ordered by decreasing accuracy at $T=4$. Rows show recurrent depth. A shared color scale preserves absolute differences across depths. Appendix~\ref{app:configuration-rankings} shows the corresponding within-depth rankings.}
\label{fig:accuracy-depth}
\end{figure}

Together with the recurrent error analysis in Appendix~\ref{app:tangent}, these observations motivate selecting shared representations by their target-depth prediction loss, formalized in Section~\ref{sec:method}.

\section{Method}
\pagelog{Method}
\label{sec:method}

\method{} formulates shared-core quantization as recurrent-depth-aware representation selection. Each candidate specifies quantized weights and transformations that remain shared across recurrent steps, while activation distributions evolve with the hidden state. We evaluate each candidate through complete execution at the intended deployment depth $T$. Given a candidate space $\mathcal{C}$ and calibration windows $\mathcal{D}$, we define the recurrent-depth-aware quantization objective as
\begin{equation}
 c_T^\star \in \arg\min_{c\in\mathcal C}\mathcal L_T(c;\mathcal D),
 \qquad
 \mathcal L_T(c;\mathcal D)=\frac{1}{|\mathcal D|}\sum_{x\in\mathcal D}L_T(c,x),
 \label{eq:selection-objective}
\end{equation}
Here, $L_T(c,x)$ denotes the mean next-token negative log-likelihood (NLL) on window $x$ after $T$ recurrent steps. This objective evaluates the quality of a shared quantized representation through its complete target-depth recurrent computation. Channel scaling and orthogonal rotations parameterize the candidate space, while progressive full-depth selection provides an efficient solution.

\subsection{Shared Quantization Representations}
\label{sec:3_1}
Each candidate $c=(s,p,b)$ specifies a rotation seed, a channel-scaling power, and a Hadamard block-size setting. The same configuration determines operator-specific transforms $D_{\ell,p}$ and $R_{\ell,s,b}$, which remain fixed across recurrent calls. For each linear operator $\ell$, the paired transforms satisfy
\begin{equation}
 xW_\ell^\top = \left(xD_{\ell,p}^{-1}R_{\ell,s,b}\right)\left(W_\ell D_{\ell,p}R_{\ell,s,b}\right)^\top.
 \label{eq:shared-representation}
\end{equation}
In exact arithmetic, these transforms preserve the recurrent hidden-state trajectory (Appendix~\ref{app:shared-transform-equivalence}). Quantization generally breaks this equivalence, motivating selection by final prediction quality.

For each recurrent linear operator $\ell$, we obtain a positive channel-scaling vector $d_\ell$ from calibration and construct
\begin{equation}
 D_{\ell,p}=\operatorname{diag}\!\left(\frac{d_\ell^{\,p}}{\sqrt{\max(d_\ell^{\,p})\min(d_\ell^{\,p})}}\right),
 \label{eq:scale-power}
\end{equation}
\begin{equation}
 R_{\ell,s,b}=S_{\ell,s}^{\mathrm{pre}}\operatorname{blkdiag}(H_b,\ldots,H_b)S_{\ell,s}^{\mathrm{post}},
 \label{eq:candidate}
\end{equation}
where powers are applied elementwise, $H_b$ is a normalized Hadamard matrix of size $b$, and $S_{\ell,s}^{\mathrm{pre}}$ and $S_{\ell,s}^{\mathrm{post}}$ are diagonal sign matrices generated deterministically from the seed and operator identity. We use randomized Hadamard rotations~\citep{ashkboos2024quarot}, restricting $b$ to supported powers of two that divide the operator's input dimension. Search varies $(s,p,b)$ while keeping the calibrated scaling vectors fixed.

For each candidate, we quantize the transformed weights and dynamically quantize the transformed inputs, yielding
\begin{equation}
 \widehat f_{\ell,c}(x)=Q_a\!\left(xD_{\ell,p}^{-1}R_{\ell,s,b}\right)Q_w\!\left(W_\ell D_{\ell,p}R_{\ell,s,b}\right)^\top,
 \label{eq:quantized-operator}
\end{equation}
Here, $Q_w$ and $Q_a$ denote weight and activation quantization. Weight quantization is fixed for each candidate, whereas activation scales are computed dynamically per token. The transforms act locally at linear operators, leaving residual-stream coordinates unchanged and requiring no step-specific weight copies. Figure~\ref{fig:operator-transforms} illustrates these paired transforms in attention and FFN operators. 

\begin{figure}[t] \centering \includegraphics[width=0.75\linewidth]{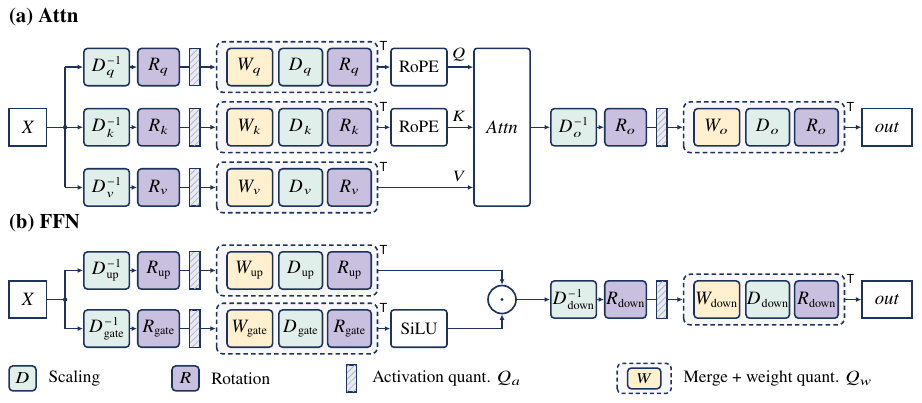} \caption{\textbf{Operator-wise quantization transforms}. Channel scaling and rotation in attention and FFN, with weight-side transforms merged before quantization and input transforms applied online.} 
\label{fig:operator-transforms} 
\end{figure} 

\subsection{Progressive Full-Depth Selection}
Using fewer calibration windows estimates the same target-depth objective, whereas reducing recurrent depth changes the computation being evaluated. \method{} therefore progressively allocates windows to surviving candidates while preserving complete target-depth execution (Figure~\ref{fig:loopy-overview}).\par

Following successive halving~\citep{jamieson2016nonstochastic}, we construct nested calibration subsets $\mathcal D_1\subset\cdots\subset\mathcal D_R=\mathcal D$ and initialize $\mathcal C_1=\mathcal C$. At stage $r$, all active candidates are evaluated on the same windows through complete $T$-step execution:
\begin{equation}
 \widehat{\mathcal L}_r(c)=\frac{1}{|\mathcal D_r|}\sum_{x\in\mathcal D_r}L_T(c,x),\qquad c\in\mathcal C_r.
 \label{eq:progressive-evaluation}
\end{equation}
We reuse previously computed losses and evaluate only newly added windows at each stage.

After each intermediate stage, we retain the $K_{r+1}$ candidates with the lowest estimated loss,
\begin{equation}
 \mathcal C_{r+1}=\operatorname{Lowest}_{K_{r+1}}\!\left(\mathcal C_r;\widehat{\mathcal L}_r\right),\qquad r=1,\ldots,R-1,
 \label{eq:candidate-retention}
\end{equation}
where the sample budgets and retained candidate counts are specified before search. All candidates within a stage are compared on the same calibration windows, and ties are resolved using a fixed candidate order. As the active set shrinks, additional windows refine the comparison among the remaining configurations. At the final stage, every surviving candidate is evaluated on the full calibration set, and we select
\begin{equation}
 \hat c_T=\arg\min_{c\in\mathcal C_R}\mathcal L_T(c;\mathcal D).
 \label{eq:final-selection}
\end{equation}
Early elimination may discard the full-search optimum. A sufficient condition for preserving it is that, at each screening stage, the full-set loss gap between the optimum and candidates beyond the retention cutoff exceeds twice the maximum estimation error (Appendix~\ref{app:optimal-candidate-preservation}). This condition motivates allocating more windows to refine candidate comparisons while retaining enough candidates to accommodate uncertainty. We assess the resulting approximation empirically through selection regret and prediction quality on separate evaluation data.

Let $n_r=|\mathcal D_r|$, with $n_0=0$. By reusing previously computed losses, \method{} requires
\begin{equation}
 N_{\mathrm{eval}}=\sum_{r=1}^{R}|\mathcal C_r|(n_r-n_{r-1})
 \label{eq:evaluation-cost}
\end{equation}
candidate--window evaluations, compared with $|\mathcal C|\,|\mathcal D|$ for exhaustive search. This count measures evaluation work, and wall-clock cost also includes candidate construction and scheduling. At deployment, the selected transformed weights are precomputed and quantized, while input transforms and dynamic activation quantization are applied during execution. Algorithm~\ref{alg:loopy-selection} summarizes the selection procedure. In Algorithm~\ref{alg:loopy-selection}, \texttt{BuildOrRestore} constructs or restores candidates from fixed calibration bases, and \texttt{EvalFullDepth} sums window-level mean NLLs after $T$ recurrent steps.

\begin{algorithm}[t]
\caption{Loopy Progressive Full-Depth Selection}
\label{alg:loopy-selection}
\begin{algorithmic}[1]
\Require $\mathcal C$, $T$, nested subsets $\{\mathcal D_r\}_{r=1}^{R}$, retained counts $\{K_r\}_{r=2}^{R}$
\State $\mathcal C_1\gets\mathcal C$, $\mathcal D_0\gets\varnothing$, $S(c)\gets0$ for all $c\in\mathcal C$
\For{$r=1,\ldots,R$}
    \For{$c\in\mathcal C_r$}
        \State $\widehat F_c\gets\Call{BuildOrRestore}{c}$
        \State $S(c)\gets S(c)+\Call{EvalFullDepth}{\widehat F_c,T,\mathcal D_r\setminus\mathcal D_{r-1}}$
        \State $\widehat{\mathcal L}_r(c)\gets S(c)/|\mathcal D_r|$
    \EndFor
    \If{$r<R$}
        \State $\mathcal C_{r+1}\gets\Call{KeepLowest}{\mathcal C_r,\widehat{\mathcal L}_r,K_{r+1}}$
    \EndIf
\EndFor
\State $\hat c_T\gets\Call{Best}{\mathcal C_R,\widehat{\mathcal L}_R}$
\State \Return $\Call{BuildOrRestore}{\hat c_T}$
\end{algorithmic}
\end{algorithm}

\subsection{Target-Depth Selection Analysis}
\label{sec:selection-analysis}

Recurrent error propagation can reverse candidate rankings across depths, even under shared dynamics (Appendix~\ref{app:contractive-counterexample}). We therefore distinguish depth mismatch from finite-window estimation error.

Fix the candidate pool $\mathcal C$ and calibration set $\mathcal D$,
and abbreviate $\mathcal L_t(c;\mathcal D)$ as $\mathcal L_t(c)$.
Define the pairwise depth mismatch
\begin{equation}
\beta_{\tau,T}
=
\max_{c,c'\in\mathcal C}
\left|
[\mathcal L_\tau(c)-\mathcal L_\tau(c')]
-
[\mathcal L_T(c)-\mathcal L_T(c')]
\right|.
\label{eq:depth-mismatch}
\end{equation}
This measures changes in candidate comparisons rather than
candidate-independent loss shifts.

\noindent\textbf{Selection-regret bound.}
Consider nested top-$K$ screening of fixed candidates, with stage
$r<R$ evaluated at depth $\tau_r$. Suppose
$\max_{c\in\mathcal C_r}
|\widehat{\mathcal L}_r(c)-\mathcal L_{\tau_r}(c)|
\leq\varepsilon_r$,
and final selection minimizes $\mathcal L_T$ over the survivors
using all of $\mathcal D$. Then
\begin{equation}
\begin{aligned}
\operatorname{Reg}_T(\hat c_T)
&:= \mathcal L_T(\hat c_T)
-\min_{c\in\mathcal C}\mathcal L_T(c) \\
&\leq \max_{r<R}
\bigl(\beta_{\tau_r,T}+2\varepsilon_r\bigr).
\end{aligned}
\label{eq:selection-regret-bound}
\end{equation}
The bound does not require a unique optimum or its survival.
If the last global optimum is discarded, every candidate retained
at that stage satisfies the bound, as do all subsequent survivors.

For \method{}, $\tau_r=T$ and $\beta_{T,T}=0$, giving
\begin{equation}
\operatorname{Reg}_T(\hat c_T)
\leq 2\max_{r<R}\varepsilon_r.
\label{eq:full-depth-regret}
\end{equation}
Full-depth screening removes depth mismatch, leaving window-estimation error. This conditional bound applies to the fixed calibration set. Our schedule does not certify $\varepsilon_r$, and the bound guarantees neither held-out generalization nor superiority at equal compute. Appendix~\ref{app:selection-regret} proves the regret bound, and Appendix~\ref{app:optimal-candidate-preservation} gives a sufficient condition for exact optimum recovery.

\section{Experiments}
\pagelog{Experiments}\label{sec:experiments}
\subsection{Setup}\label{sec:setup}
\textbf{Models and evaluation.}
We evaluate Ouro-1.4B and Ouro-2.6B at recurrent depth $T=4$~\citep{zhu2025scaling}, Huginn-0125 at $T=32$~\citep{geiping2026scaling}, and Recurrent-Llama-3.2-1B at $T=16$~\citep{mcleish2025teaching}.
We report perplexity on WikiText-2~\citep{merity2016pointer} and LAMBADA~\citep{paperno2016lambada}, and zero-shot accuracy on HellaSwag~\citep{zellers2019hellaswag}, WinoGrande~\citep{sakaguchi2020winogrande}, LAMBADA, ARC-Challenge~\citep{clark2018arc}, and MMLU~\citep{hendrycks2021mmlu}. Avg. denotes the unweighted mean of the five task accuracies.

\textbf{Quantization and baselines.} We evaluate W4A4 and W4A8 against recurrent adaptations of AWQ, SmoothQuant, QuaRot, SpinQuant, and FlatQuant~\citep{lin2024awq,xiao2023smoothquant,ashkboos2024quarot,liu2025spinquant,sun2024flatquant}. \method{} and all evaluated baselines uses symmetric group-32 quantization and symmetric per-token activation quantization. Baseline-specific quantization settings and calibration procedures are detailed in Appendix~\ref{app:protocol}. LoopQ comparison uses the same Loopy results as the main table and published LoopQ results~\citep{fang2026loopq}.

\textbf{Calibration and selection.}
\method{} selects shared rotation and scaling configurations using progressive full-depth evaluation (Section~\ref{sec:method}). Calibration bases remain fixed across candidates; nested window subsets and retained candidate counts are specified before search.
All \method{} results in Table~\ref{tab:main-results} and the full-results appendix are obtained using configurations selected by the progressive full-depth procedure described in Section~\ref{sec:method}.

\textbf{Implementation.}
Main accuracy results use quantization and dequantization in BF16 with KV caching disabled. We implement packed low-bit inference in Triton~\citep{tillet2019triton}, with hardware benchmarks described in Section~\ref{sec:deployment}.

\subsection{Main Results}
\begin{table}[t]
\caption{\textbf{W4A8 and W4A4 results on looped language models.} WT2 and LMB denote WikiText-2 and LAMBADA perplexity ($\downarrow$); Avg. is mean task accuracy (\%, $\uparrow$). Detailed accuracy results are shown in Tables~\ref{tab:w4a8-full} and~\ref{tab:w4a4-full}. Bold indicates the best low-bit result per model and precision.}
\label{tab:main-results}\label{tab:w4a8}\label{tab:w4a4}
\centering\footnotesize\setlength{\tabcolsep}{1.5pt}
\begin{tabularx}{\textwidth}{@{}l*{4}{rr>{\raggedleft\arraybackslash}X}@{}}
\toprule
& \multicolumn{3}{c}{Ouro-2.6B} & \multicolumn{3}{c}{Ouro-1.4B} & \multicolumn{3}{c}{Huginn-0125} & \multicolumn{3}{c}{Recurrent-Llama-3.2-1B} \\
\cmidrule(lr){2-4}\cmidrule(lr){5-7}\cmidrule(lr){8-10}\cmidrule(l){11-13}
Method & WT2 & LMB & Avg. & WT2 & LMB & Avg. & WT2 & LMB & Avg. & WT2 & LMB & Avg. \\
\midrule
FP ref. & 10.51 & 4.23 & 69.65 & 12.03 & 5.29 & 64.66 & 13.21 & 7.30 & 50.42 & 20.72 & 17.23 & 38.63 \\
\midrule
\multicolumn{13}{l}{\textbf{W4A8}} \\  
AWQ & 11.46 & 5.22 & 66.45 & 13.64 & 7.68 & 60.51 & 14.14 & 8.85 & 47.37 & 28.16 & 50.70 & 35.10 \\
SmoothQuant & 19.44 & 15.20 & 56.15 & 53.27 & 134.60 & 37.78 & --- & --- & 25.82 & 243.77 & --- & 26.40 \\
QuaRot & 11.76 & 5.30 & 66.76 & 14.25 & 8.29 & 59.09 & --- & --- & 25.61 & 26.03 & 31.15 & 36.60 \\
SpinQuant & 12.07 & 5.56 & 66.21 & 14.18 & 7.52 & 60.42 & --- & --- & 24.75 & 25.96 & 28.42 & 36.75 \\
FlatQuant & 14.52 & 8.14 & 62.42 & 26.55 & 23.50 & 46.21 & --- & --- & 24.43 & 86.73 & --- & 28.22 \\
\rowcolor[gray]{0.93}
Loopy & \textbf{10.97} & \textbf{4.46} & \textbf{68.64} & \textbf{12.77} & \textbf{6.00} & \textbf{63.00} & \textbf{13.68} & \textbf{7.98} & \textbf{49.21} & \textbf{24.76} & \textbf{26.97} & \textbf{37.23} \\
\midrule
\multicolumn{13}{l}{\textbf{W4A4}} \\  
AWQ & 340.35 & --- & 26.0 & 645.60 & --- & 25.6 & 168.39 & 117.12 & 34.2 & --- & --- & 25.0 \\
SmoothQuant & --- & --- & 24.5 & --- & --- & 25.2 & --- & --- & 25.3 & --- & --- & 25.5 \\
QuaRot & 15.27 & 9.73 & 56.9 & 19.16 & 22.18 & 44.7 & --- & --- & 25.3 & 48.08 & 139.25 & 31.5 \\
SpinQuant & 16.62 & 8.18 & 60.9 & 17.45 & 12.86 & 52.1 & --- & --- & 25.0 & \textbf{41.03} & \textbf{102.34} & \textbf{32.5} \\
FlatQuant & --- & --- & 24.6 & --- & --- & 25.0 & --- & --- & 25.0 & --- & --- & 25.1 \\
\rowcolor[gray]{0.93}
Loopy & \textbf{12.35} & \textbf{6.20} & \textbf{64.37} & \textbf{14.73} & \textbf{8.17} & \textbf{57.75} & \textbf{28.09} & \textbf{36.36} & \textbf{37.32} & 42.08 & 110.83 & 32.0 \\
\bottomrule
\end{tabularx}\par\smallskip {\raggedright\footnotesize --- denotes perplexity greater than $10^3$.\par}
\end{table}
\textbf{W4A8.}  In Table~\ref{tab:main-results}, \method{} achieves the lowest WikiText-2 and LAMBADA perplexity and the highest available average accuracy among the evaluated low-bit methods on all four checkpoints.  \par\textbf{W4A4.} \method{} achieves the lowest perplexity on both datasets for the two Ouro checkpoints and Huginn-0125. Average accuracy exceeds the evaluated low-bit baselines on these three checkpoints. Recurrent-Llama exhibits greater degradation under W4A4 than W4A8, compared SpinQuant under W4A4, which is one of our work's limitation. Its inherited Llama pre-normalization structure differs from the normalization schemes used by Ouro and Huginn~\citep{mcleish2025teaching,zhu2025scaling,geiping2026scaling}, potentially affecting the propagation of quantization perturbations through recurrent execution. \par 
\textbf{Comparison with LoopQ.}
All LoopQ results are taken from its original paper~\citep{fang2026loopq}, while Loopy results are reused from Table~\ref{tab:main-results}. Table~\ref{tab:loopq-comparison} shows that Loopy achieves lower W4A4 WikiText-2 perplexity and higher five-task average accuracy on both Ouro checkpoints.

\begin{table}[t]
\centering\small
\caption{W4A4 comparison with LoopQ. $\dagger$: published results~\citep{fang2026loopq}; Avg. Acc. averages the five tasks in Table~\ref{tab:main-results}.}
\label{tab:loopq-comparison}
\begin{tabular*}{0.85\linewidth}{@{\extracolsep{\fill}}lrrrr@{}}
\toprule
& \multicolumn{2}{c}{WT2 PPL $\downarrow$} & \multicolumn{2}{c}{Avg. Acc. (\%) $\uparrow$} \\ \cmidrule(lr){2-3}\cmidrule(lr){4-5} Model & LoopQ$^{\dagger}$ & \method{} & LoopQ$^{\dagger}$ & \method{} \\
\midrule
Ouro-1.4B & 15.81 & \textbf{14.73} & 57.08 & \textbf{57.75} \\
Ouro-2.6B & 15.89 & \textbf{12.35} & 58.13 & \textbf{64.37} \\
\bottomrule
\end{tabular*}
\end{table}

\subsection{Selection Quality–Calibration Cost Trade-off}
\label{sec:quality_cost_new}
We compare LOOPY with full-depth uniform selection at evaluation depth \(T\), shallow progressive selection at \(T/2\), and shallow-to-full selection with evaluation depths \(T/4 \rightarrow T/2 \rightarrow T\). Experiments use Ouro-1.4B W4A4 (\(T=4\)) and Huginn-0125 W4A8 (\(T=32\)). Full-depth uniform selection allocates a finite calibration budget equally across candidates, while the other strategies allocate additional calibration windows according to their respective selection procedures. We measure the end-to-end calibration cost of each strategy, including base calibration, candidate construction or restoration, candidate screening, and final model construction.

Within each model, candidate representations and quantization settings are fixed. We repeat selection using \(N=10\) independently sampled calibration sets, with all strategies sharing the same ordered calibration data within each run. Each strategy is evaluated under multiple predefined calibration budgets, and the selected configurations are evaluated at deployment depth \(T\) on held-out WikiText-2 using token-normalized perplexity with 2048-token windows for convenience. Figure~6 presents the relationship between calibration cost and selection quality, where lower perplexity indicates better selected configurations. Figure \ref{fig:selection_cost} shows that \method{} achieves a favorable quality–cost trade-off across both models. Compared with full-depth exhaustive search, LOOPY reaches comparable selection quality with substantially lower calibration cost. On Ouro-1.4B W4A4 and Huginn-0125 W4A8, LOOPY reduces end-to-end calibration cost by 4.2× and 3.5×, respectively.

\begin{figure}[!t]
    \centering
    \includegraphics[width=\linewidth]{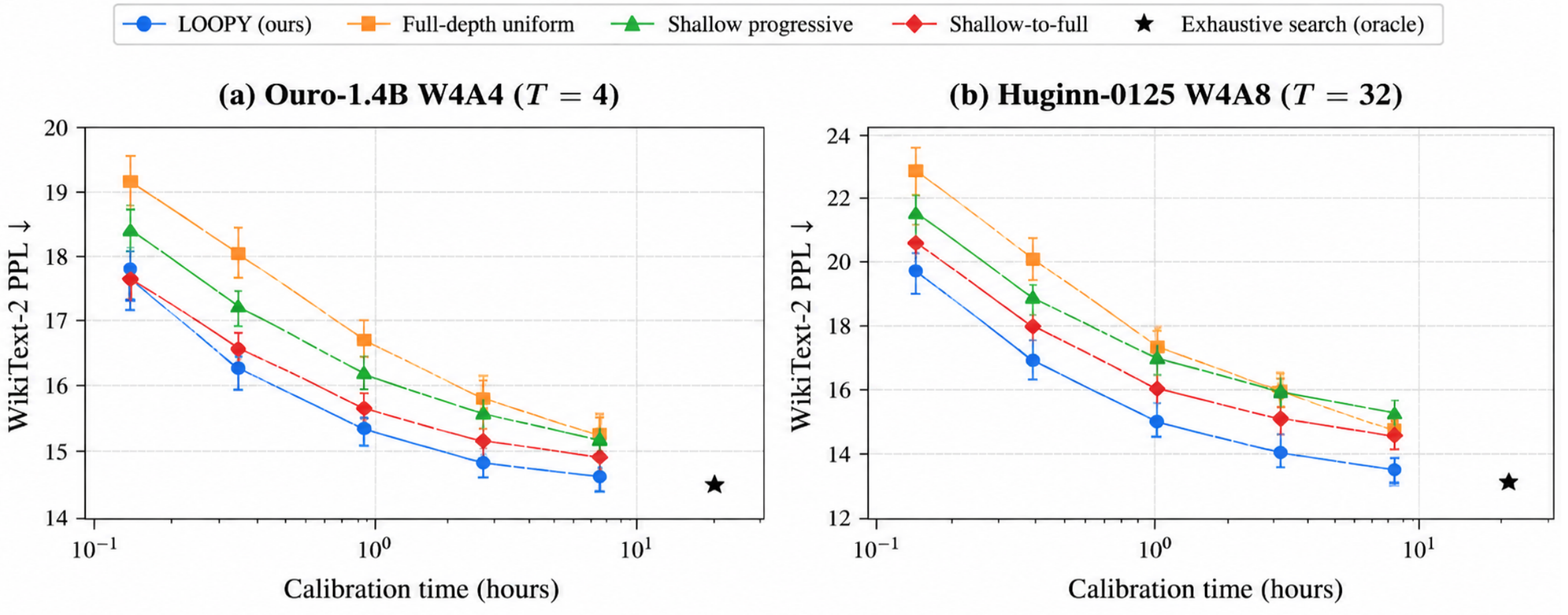}
    \caption{\textbf{Selection quality--calibration cost trade-off.} WikiText-2 perplexity versus calibration time for (a) Ouro-1.4B W4A4 ($T=4$) and (b) Huginn-0125 W4A8 ($T=32$). Black stars denote exhaustive-search references. Lower perplexity is better.}
    \label{fig:selection_cost}
\end{figure}

\subsection{KV Cache Quantization}\label{sec:kv-quantization}
We evaluate 4-bit and 8-bit KV quantization on Ouro-1.4B at $T=4$ with fixed Loopy W4A4 and W4A8 configurations, optionally applying post-RoPE Q/K Hadamard transforms following SpinQuant~\citep{liu2025spinquant}. We report WikiText-2 perplexity and average accuracy over four downstream tasks, with the full evaluation protocol and per-task results provided in Appendix~\ref{app:kv-results}.

As shown in Table~\ref{tab:kv-quantization}, Q/K Hadamard transforms reduce the relative perplexity increase from KV4 quantization from 6.56\% to 3.41\% under W4A4 and from 4.38\% to 2.19\% under W4A8. They also improve average accuracy by 1.40 percentage points over KV4 under W4A4.

\begin{table}[!t]
\centering\small
\caption{KV quantization results. $\Delta$PPL is relative to BF16 KV within each W/A setting; Avg. Acc. averages four tasks. $H$ denotes the Q/K Hadamard transform.}
\label{tab:kv-quantization}
\setlength{\tabcolsep}{9pt}
\begin{tabular}{llrrr}
\toprule
W/A & KV setting & PPL $\downarrow$ & $\Delta$PPL (\%) & Avg. Acc. (\%) $\uparrow$ \\
\midrule
W4A4 & BF16 KV & 11.4174 & --- & 54.5266 \\
& BF16 KV+$H$ & 11.4139 & $-0.03$ & 54.4818 \\
& KV8 & 11.4109 & $-0.06$ & 54.1423 \\
& KV4 & 12.1660 & $+6.56$ & 51.9853 \\
\rowcolor[gray]{0.93}
& KV4+$H$ & 11.8062 & $+3.41$ & 53.3883 \\
\midrule
W4A8 & BF16 KV & 9.3425 & --- & 61.7154 \\
& BF16 KV+$H$ & 9.3422 & $-0.004$ & 61.7387 \\
& KV8 & 9.3457 & $+0.03$ & 61.9927 \\
& KV4 & 9.7513 & $+4.38$ & 61.4155 \\
\rowcolor[gray]{0.93}
& KV4+$H$ & 9.5470 & $+2.19$ & 61.4134 \\
\bottomrule
\end{tabular}
\end{table}

\subsection{Deployment Efficiency and Validation}\label{sec:deployment}
\textbf{Decoding throughput.}
We evaluate Ouro-1.4B at a recurrent depth of $T=4$ on an NVIDIA H20 GPU under W4A4 and W4A8 quantization. With batch size 4 and an input length of 2048 tokens, online Loopy achieves about 41 and 41 decoding tokens per second (tokens/s), respectively, compared with 35 and 35 tokens/s for the official SpinQuant implementation. Table~\ref{tab:h20-throughput} summarizes the results, while Appendix~\ref{app:h20-throughput} details the timing protocol and implementation differences.

\begin{table}[!t]
\centering
\small
\begin{threeparttable}
\caption{\textbf{Decoding throughput on NVIDIA H20 (tokens/s).}
B/input denotes batch size and input length. Protocol details and
additional configurations appear in Appendix~\ref{app:h20-throughput}.}
\label{tab:h20-throughput}

\setlength{\tabcolsep}{9pt}
\begin{tabular}{lrrrr}
\toprule
& \multicolumn{2}{c}{W4A4} & \multicolumn{2}{c}{W4A8} \\
\cmidrule(lr){2-3}\cmidrule(lr){4-5}
Method & B1/128 & B4/2048 & B1/128 & B4/2048 \\
\midrule
QuaRot & $9.08^{*}$ & $35.14^{*}$ & 8.80 & 34.75 \\
SpinQuant & $9.16^{*}$ & $35.85^{*}$ & 9.14 & 35.78 \\
\rowcolor[gray]{0.93}
Loopy (online) & 10.01 & 40.93 & 10.22 & 41.17 \\
\rowcolor[gray]{0.93}
Loopy (offline)\tnote{\textdagger} & \textbf{11.21} & \textbf{45.45} & \textbf{11.37} & \textbf{45.83} \\
\bottomrule
\end{tabular}

\begin{tablenotes}[flushleft]
\scriptsize
\item[$\dagger$] Fully absorbable transform configuration.
\item[$*$] Numerical-mismatch diagnostic runs.
\end{tablenotes}

\end{threeparttable}
\end{table}


\section{Related Work}
\label{sec:related-work}
\textbf{Looped language models.} Looped Transformers increase effective depth through repeated application of shared parameters, as explored by Universal Transformers and subsequent looped language models \citep{dehghani2019universal,saunshi2025reasoning,geiping2026scaling,zhu2025scaling,jeddi2026loopformer}. Parameter sharing reduces the number of distinct weights required to represent a deep computation, motivating studies of memory-efficient recurrent architectures \citep{vendrell2026memory}.

\textbf{Post-training quantization.} Existing methods reduce quantization error through weight reconstruction or channel scaling, as exemplified by GPTQ, AWQ, and SmoothQuant \citep{frantar2022gptq,lin2024awq,xiao2023smoothquant}. Complementing these approaches, rotation-based methods redistribute outliers, with QuaRot using structured orthogonal transformations and SpinQuant learning rotations through an end-to-end quantization objective \citep{ashkboos2024quarot,liu2025spinquant}.

\textbf{Recurrence-aware optimization.} Repeated execution of the recurrent core produces evolving hidden states, allowing quantization errors to propagate across recurrent steps. LoopQ addresses this problem through loop-aware transformations and trajectory-aware calibration \citep{fang2026loopq}.

\section{Conclusion}
\pagelog{Conclusion}\label{sec:conclusion}
We presented \method{}, a PTQ framework that targets looped language models. By selecting shared low-bit representations according to their final target-depth prediction loss, \method{} aligns quantization selection with recurrent deployment and makes this selection practical through progressive full-depth evaluation. Experiments show benefits varying across architectures.

\label{main-text-end}
\clearpage
\section*{AI Use Statement}

Generative AI tools were used to assist with manuscript organization,
language editing, literature checking, and the development and execution
of code for manuscript preparation, verification, and supplementary
experiments. The authors take full responsibility for the accuracy and
integrity of the paper, including its claims, experimental results,
references, and supplementary materials.
\subsection*{Reproducibility statement}
The supplementary source includes the numerical inputs, figure-generation code,
source hashes, and a claim-to-evidence ledger. Appendix~\ref{app:protocol}
describes the evaluation protocol, while Appendix~\ref{app:provenance} documents the supporting records for quantization evaluation, diagnostic analysis, and packed-execution validation.
\pagelog{References}
\bibliography{references}
\bibliographystyle{iclr2027_conference}
\appendix
\clearpage
\section*{Appendix}

\section{Full Results}\label{app:full-results}
Tables~\ref{tab:w4a8-full} and \ref{tab:w4a4-full} report the full results. The remaining metrics retain their original evaluation. Avg. is the mean of five task accuracies (\%). FP ref. provides full-precision context. Collapsed results are retained. W4A8 averages reproduce the supplied results.

\begin{table}[!htbp]
\caption{Complete W4A8 results. HS: HellaSwag; WG: WinoGrande; LMB: LAMBADA.
$\dagger$: adapted baselines.}
\label{tab:w4a8-full}\centering\fontsize{9}{11}\selectfont\setlength{\tabcolsep}{2pt}\setlength{\aboverulesep}{1pt}\setlength{\belowrulesep}{1pt}
\begin{tabularx}{\textwidth}{@{}llrr*{6}{>{\raggedleft\arraybackslash}X}@{}}\toprule
Model & Method & WT2 PPL & LMB PPL & HS & WG & LMB Acc. & ARC-C & MMLU & Avg.\\\midrule
\multirow{7}{*}{\shortstack[l]{Ouro-\\2.6B}} & FP ref. & 10.51 & 4.23 & 76.45 & 69.53 & 70.08 & 58.36 & 73.83 & 69.65 \\ 
 & AWQ & 11.4603 & 5.2171 & 74.71 & 67.80 & 64.22 & 55.89 & 69.63 & 66.45 \\
 & SmoothQuant$^{\dagger}$ & 19.4442 & 15.2011 & 68.29 & 61.96 & 45.37 & 50.26 & 54.85 & 56.15 \\
 & QuaRot$^{\dagger}$ & 11.7643 & 5.3048 & 74.13 & 68.43 & 64.70 & 56.83 & 69.69 & 66.76 \\
 & SpinQuant$^{\dagger}$ & 12.0654 & 5.5639 & 74.28 & 66.93 & 63.13 & 57.76 & 68.93 & 66.21 \\
 & FlatQuant$^{\dagger}$ & 14.5179 & 8.1443 & 71.83 & 66.69 & 56.12 & 54.10 & 63.35 & 62.42 \\
\rowcolor[gray]{0.93}
 & Loopy & 10.97 & 4.46 & 75.94 & 70.48 & 68.76 & 56.14 & 71.86 & 68.64 \\
\midrule
\multirow{7}{*}{\shortstack[l]{Ouro-\\1.4B}} & FP ref. & 12.03 & 5.29 & 71.69 & 66.06 & 65.05 & 52.30 & 68.20 & 64.66 \\ 
 & AWQ & 13.6390 & 7.6843 & 68.43 & 63.85 & 56.72 & 51.71 & 61.85 & 60.51 \\
 & SmoothQuant$^{\dagger}$ & 53.2679 & 134.5964 & 52.44 & 56.91 & 19.44 & 34.73 & 25.40 & 37.78 \\
 & QuaRot$^{\dagger}$ & 14.2539 & 8.2905 & 68.71 & 64.01 & 55.85 & 48.46 & 58.39 & 59.09 \\
 & SpinQuant$^{\dagger}$ & 14.1781 & 7.5221 & 69.31 & 63.77 & 58.35 & 50.00 & 60.66 & 60.42 \\
 & FlatQuant$^{\dagger}$ & 26.5550 & 23.5027 & 62.40 & 59.04 & 38.85 & 45.56 & 25.22 & 46.21 \\
\rowcolor[gray]{0.93}
 & Loopy & 12.77 & 6.00 & 70.38 & 64.33 & 62.12 & 52.05 & 66.13 & 63.00 \\
\midrule
\multirow{7}{*}{\shortstack[l]{Huginn-\\0125}} & FP ref. & 13.21 & 7.30 & 66.68 & 61.64 & 55.35 & 37.54 & 30.90 & 50.42 \\ 
 & AWQ & 14.1421 & 8.8549 & 63.78 & 59.98 & 51.56 & 35.67 & 25.84 & 47.37 \\
 & SmoothQuant$^{\dagger}$ & $5.32\!\times\!10^{4}$ & $7.76\!\times\!10^{6}$ & 26.59 & 51.30 & 0.00 & 26.28 & 24.95 & 25.82 \\
 & QuaRot$^{\dagger}$ & $6.41\!\times\!10^{3}$ & $5.81\!\times\!10^{5}$ & 26.90 & 52.33 & 0.02 & 25.85 & 22.97 & 25.61 \\
 & SpinQuant$^{\dagger}$ & $6.65\!\times\!10^{3}$ & $5.03\!\times\!10^{5}$ & 26.56 & 49.33 & 0.04 & 23.98 & 23.83 & 24.75 \\
 & FlatQuant$^{\dagger}$ & $2.70\!\times\!10^{4}$ & $3.76\!\times\!10^{6}$ & 25.97 & 48.15 & 0.00 & 25.17 & 22.86 & 24.43 \\
\rowcolor[gray]{0.93}
 & Loopy & 13.68 & 7.98 & 64.71 & 61.17 & 53.81 & 36.69 & 29.67 & 49.21 \\
\midrule
\multirow{7}{*}{\shortstack[l]{Recurrent-\\Llama-\\3.2-1B}} & FP ref. & 20.72 & 17.23 & 47.51 & 52.49 & 40.81 & 29.35 & 23.00 & 38.63 \\ 
 & AWQ & 28.1567 & 50.7029 & 44.58 & 52.41 & 27.17 & 28.24 & 23.11 & 35.10 \\
 & SmoothQuant$^{\dagger}$ & 243.7745 & $1.53\!\times\!10^{4}$ & 32.75 & 49.09 & 3.30 & 23.63 & 23.22 & 26.40 \\
 & QuaRot$^{\dagger}$ & 26.0300 & 31.1539 & 43.77 & 53.83 & 32.06 & 28.41 & 24.95 & 36.60 \\
 & SpinQuant$^{\dagger}$ & 25.9630 & 28.4209 & 45.00 & 52.88 & 34.06 & 28.58 & 23.20 & 36.75 \\
 & FlatQuant$^{\dagger}$ & 86.7286 & $1.40\!\times\!10^{3}$ & 36.31 & 49.33 & 6.71 & 25.68 & 23.08 & 28.22 \\
\rowcolor[gray]{0.93}
 & Loopy & 24.76 & 26.97 & 45.88 & 52.80 & 34.76 & 28.75 & 23.94 & 37.23 \\
\bottomrule

\end{tabularx}
\end{table}
\clearpage
\begin{table}[!htbp]
\caption{Complete W4A4 results. Avg. is the arithmetic mean of the five displayed task accuracies (\%). Large perplexities are shown in scientific notation.}
\label{tab:w4a4-full}\centering\fontsize{9}{11}\selectfont\setlength{\tabcolsep}{2pt}\setlength{\aboverulesep}{1pt}\setlength{\belowrulesep}{1pt}
\begin{tabularx}{\textwidth}{@{}llrr*{6}{>{\raggedleft\arraybackslash}X}@{}}\toprule
Model & Method & WT2 PPL & LMB PPL & HS & WG & LMB Acc. & ARC-C & MMLU & Avg.\\\midrule
\multirow{7}{*}{\shortstack[l]{Ouro-\\2.6B}} & FP ref. & 10.51 & 4.23 & 76.45 & 69.53 & 70.08 & 58.36 & 73.83 & 69.65 \\ 
 & AWQ & 340.35 & $1.35\!\times\!10^{4}$ & 30.5 & 49.7 & 1.9 & 24.5 & 23.4 & 26.00 \\
 & SmoothQuant & $1.01\!\times\!10^{4}$ & $1.25\!\times\!10^{6}$ & 26.3 & 48.4 & 0.0 & 24.7 & 23.2 & 24.52 \\
 & QuaRot & 15.27 & 9.73 & 69.1 & 60.9 & 52.1 & 49.2 & 53.4 & 56.94 \\
 & SpinQuant & 16.62 & 8.18 & 70.4 & 63.1 & 56.2 & 53.5 & 61.2 & 60.88 \\
 & FlatQuant & $1.57\!\times\!10^{4}$ & $3.77\!\times\!10^{6}$ & 26.2 & 47.9 & 0.0 & 24.9 & 23.8 & 24.56 \\
\rowcolor[gray]{0.93}
 & Loopy & 12.35 & 6.20 & 72.87 & 65.59 & 62.35 & 55.38 & 65.64 & 64.37 \\
\midrule
\multirow{7}{*}{\shortstack[l]{Ouro-\\1.4B}} & FP ref. & 12.03 & 5.29 & 71.69 & 66.06 & 65.05 & 52.30 & 68.20 & 64.66 \\ 
 & AWQ & 645.60 & $7.02\!\times\!10^{4}$ & 29.7 & 50.6 & 0.4 & 23.4 & 23.7 & 25.56 \\
 & SmoothQuant & $2.22\!\times\!10^{4}$ & $9.83\!\times\!10^{6}$ & 26.7 & 49.1 & 0.0 & 25.4 & 24.8 & 25.20 \\
 & QuaRot & 19.16 & 22.18 & 62.0 & 55.5 & 39.0 & 40.7 & 26.2 & 44.68 \\
 & SpinQuant & 17.45 & 12.86 & 64.4 & 59.9 & 47.9 & 42.1 & 46.2 & 52.10 \\
 & FlatQuant & $4.69\!\times\!10^{4}$ & $8.32\!\times\!10^{6}$ & 26.2 & 49.0 & 0.0 & 24.1 & 25.8 & 25.02 \\
\rowcolor[gray]{0.93}
 & Loopy & 14.73 & 8.17 & 67.70 & 61.48 & 55.81 & 47.70 & 56.07 & 57.75 \\
\midrule
\multirow{7}{*}{\shortstack[l]{Huginn-\\0125}} & FP ref. & 13.21 & 7.30 & 66.68 & 61.64 & 55.35 & 37.54 & 30.90 & 50.42 \\ 
 & AWQ & 168.39 & 117.12 & 37.3 & 50.4 & 31.4 & 27.5 & 24.4 & 34.20 \\
 & SmoothQuant & $8.46\!\times\!10^{4}$ & $7.60\!\times\!10^{6}$ & 26.3 & 48.5 & 0.0 & 26.2 & 25.4 & 25.28 \\
 & QuaRot & $1.82\!\times\!10^{4}$ & $1.09\!\times\!10^{6}$ & 26.3 & 50.7 & 0.0 & 26.1 & 23.6 & 25.34 \\
 & SpinQuant & $1.62\!\times\!10^{4}$ & $1.14\!\times\!10^{6}$ & 26.5 & 50.0 & 0.0 & 24.1 & 24.3 & 24.98 \\
 & FlatQuant & $6.22\!\times\!10^{4}$ & $5.75\!\times\!10^{6}$ & 26.2 & 48.4 & 0.0 & 26.3 & 24.3 & 25.04 \\
\rowcolor[gray]{0.93}
 & Loopy & 28.09 & 36.36 & 47.18 & 51.46 & 33.61 & 29.61 & 24.76 & 37.32 \\
\midrule
\multirow{7}{*}{\shortstack[l]{Recurrent-\\Llama-\\3.2-1B}} & FP ref. & 20.72 & 17.23 & 47.51 & 52.49 & 40.81 & 29.35 & 23.00 & 38.63 \\ 
 & AWQ & $6.37\!\times\!10^{3}$ & $1.14\!\times\!10^{6}$ & 26.1 & 49.6 & 0.0 & 25.1 & 24.1 & 24.98 \\
 & SmoothQuant & $5.59\!\times\!10^{4}$ & $1.34\!\times\!10^{7}$ & 26.2 & 51.1 & 0.0 & 25.0 & 25.4 & 25.54 \\
 & QuaRot & 48.08 & 139.25 & 38.0 & 52.7 & 17.8 & 25.4 & 23.8 & 31.54 \\
 & SpinQuant & 41.03 & 102.34 & 39.7 & 50.5 & 21.5 & 27.2 & 23.6 & 32.50 \\
 & FlatQuant & $1.04\!\times\!10^{4}$ & $5.18\!\times\!10^{6}$ & 26.0 & 51.0 & 0.0 & 25.5 & 23.2 & 25.14 \\
\rowcolor[gray]{0.93}
 & Loopy & 42.08 & 110.83 & 37.68 & 52.17 & 20.64 & 26.70 & 23.07 & 32.0 \\
\bottomrule

\end{tabularx}\end{table}
\clearpage

\subsection{KV Quantization Protocol and Per-task Results}\label{app:kv-results}
We evaluate ByteDance/Ouro-1.4B at $T=4$, freezing the existing Loopy W4A4 and W4A8 weights and transforms without additional search. The controls are BF16 KV, BF16 KV+$H$, KV8, KV4, and KV4+$H$. Evaluation uses full-sequence teacher forcing with \texttt{use\_cache=False}. K and V are quantized and dequantized for attention computation; these experiments assess prediction quality rather than cached decoding performance.

Following SpinQuant~\citep{liu2025spinquant}, the $+H$ variants apply the same normalized $H_{128}$ to Q and K after RoPE and before the attention dot product. K is quantized after this transform; Q receives no additional quantization, and V receives no additional rotation. Otherwise, K is quantized directly after RoPE. V is quantized after its projection and head partitioning. Both K and V use dynamic asymmetric min--max quantization with a separate scale and zero-point for each token and head (128 dimensions), including zero in the range and applying no clipping. KV8 and KV4 quantize both tensors to 8 and 4 bits, respectively. Statistics and quantization arithmetic use FP32, and dequantized outputs are cast to BF16.

WikiText-2 raw test text is concatenated and partitioned into non-overlapping blocks of up to 2048 tokens, retaining the final shorter block. The first token of each block is excluded from loss. We aggregate NLL weighted by the number of valid prediction tokens and report $\mathrm{PPL}=\exp(\mathrm{NLL})$, over 149 blocks and 304,829 valid prediction tokens. This token-normalized protocol differs from the word-normalized 4096-token protocol of the main results. Average accuracy is the unweighted mean over ARC-Challenge, WinoGrande, LAMBADA, and HellaSwag, excluding MMLU.

Table~\ref{tab:kv-tasks} reports per-task accuracy. Under W4A4, Q/K Hadamard transforms improve all four task accuracies relative to KV4, whereas under W4A8, task-level gains and losses largely offset each other.
\begin{table}[t]
\centering\small
\caption{Per-task accuracy (\%) for KV quantization. $H$ denotes the Q/K Hadamard transform.}
\label{tab:kv-tasks}
\setlength{\tabcolsep}{7pt}
\begin{tabular}{llrrrr}
\toprule
W/A & KV setting & ARC-C & WinoGrande & LAMBADA & HellaSwag \\
\midrule
W4A4 & BF16 KV & 44.20 & 62.51 & 45.97 & 65.43 \\
& BF16 KV+$H$ & 44.97 & 60.85 & 46.59 & 65.51 \\
& KV8 & 45.65 & 58.96 & 46.48 & 65.48 \\
& KV4 & 40.78 & 58.64 & 44.34 & 64.17 \\
\rowcolor[gray]{0.93}
& KV4+$H$ & 40.96 & 61.88 & 45.68 & 65.04 \\
\midrule
W4A8 & BF16 KV & 51.02 & 64.56 & 61.48 & 69.80 \\
& BF16 KV+$H$ & 50.85 & 64.88 & 61.56 & 69.67 \\
& KV8 & 50.77 & 65.75 & 61.77 & 69.69 \\
& KV4 & 51.54 & 65.11 & 59.71 & 69.30 \\
\rowcolor[gray]{0.93}
& KV4+$H$ & 50.51 & 64.56 & 61.05 & 69.53 \\
\bottomrule
\end{tabular}
\end{table}

\subsection{Lower-Bit Quantization on Ouro Models}\label{app:lower-bit} Under W3A4, Loopy achieves WikiText-2 perplexities of 23.77 and 17.65 on Ouro-1.4B and Ouro-2.6B, respectively. W3A3 substantially increases perplexity to 1653.88 and 361.85, revealing severe degradation at this more aggressive precision.\par \begin{table}[H] \centering\small \caption{\textbf{WikiText-2 perplexity under lower-bit quantization.} All quantized results use Loopy. Lower is better.} \label{tab:ouro-lower-bit} \begin{tabular}{lrrrr} \toprule Model & BF16 & W4A4 & W3A4 & W3A3 \\ \midrule Ouro-1.4B & 12.03 & 14.73 & 23.77 & 1653.88 \\ Ouro-2.6B & 10.51 & 12.35 & 17.65 & 361.85 \\ \bottomrule \end{tabular} \end{table} \section{Evaluation Protocol}\label{app:protocol}
\textbf{Model identity and scope.}
We evaluate four checkpoints at depths 4, 4, 32, and 16: Ouro-2.6B, Ouro-1.4B, Huginn-0125, and Recurrent-Llama-3.2-1B. The Ouro sizes share a family. Checkpoint names denote physical parameter count, not unrolled depth. The paper's adapters quantize eligible dense Transformer linears across the prelude, recurrent core, and coda, while Loopy transforms only recurrent-core linears, retaining the AWQ anchor elsewhere. Embeddings, normalization, LM heads, and early-exit gates remain unquantized.

\textbf{Quantization recipes.}
LOOPY and AWQ both use symmetric group-wise W4 quantization with a weight group size of 32, together with dynamic symmetric per-token A8 or A4 quantization. In the W4A4 comparison in Table~\ref{tab:main-results}, our SmoothQuant adaptation uses group-wise W4 quantization with a weight group size of 32 and dynamic A4 quantization, and our QuaRot adaptation~\citep{ashkboos2024quarot} combines graph rotations with GPTQ~\citep{frantar2022gptq}, our SpinQuant adaptation~\citep{liu2025spinquant} learns orthogonal transforms before GPTQ and uses symmetric A4. Our FlatQuant adaptation~\citep{sun2024flatquant} uses layerwise affine/diagonal transformations and clipping with W4 RTN. These recurrent-architecture adaptations use the same weight group size of 32. Activation symmetry and calibration compute are not matched across methods. Collapsed adaptations do not establish a failure of every implementation of the original method.

\textbf{LoopQ comparison.}
Table~\ref{tab:loopq-comparison} reuses the W4A4 Loopy results from Table~\ref{tab:main-results} for Ouro-1.4B and Ouro-2.6B at $T=4$. LoopQ perplexities and task accuracies are taken from Table~1 of \citet{fang2026loopq}; its average accuracy is computed as the unweighted mean of HellaSwag, WinoGrande, LAMBADA, ARC-Challenge, and MMLU.   Quantization settings are aligned, while each method retains its own calibration procedure.

\textbf{Likelihood and downstream tasks.}
The full test splits are evaluated without a sample cap, with batch size one.
The main WikiText-2 results report word-normalized perplexity with a context length of 4096 tokens, aligned with the LoopQ evaluation protocol. The calibration-depth ablation and numerical-validation experiments use token-normalized perplexity on non-overlapping 2048-token windows, excluding the first position of each window. LAMBADA perplexity and final-word accuracy follow
the evaluator's task definition. HellaSwag, WinoGrande, ARC-Challenge, and
MMLU use zero-shot answer scoring. MMLU covers all 57 subjects. The paper
runtime asserts the requested number of recurrent steps and sets
\texttt{use\_cache=False}. Huginn's request-dependent state initialization is
seeded deterministically so that resumed shards use the same state.

\subsection{Selection Quality and Calibration Cost}\label{app:selection_cost}The comparison uses Ouro-1.4B W4A4 at $T=4$ and Huginn-0125 W4A8 at $T=32$. Shallow uniform ($T/2$), full-depth uniform ($T$), shallow-to-full screening ($T/4\!\rightarrow\!T/2\!\rightarrow\!T$), and LOOPY share the candidate pool, calibrated scaling bases, and quantization recipe within each model.   Selection regret is evaluated on the full selection set relative to the same-pool exhaustive optimum at target depth. End-to-end calibration time includes base calibration, candidate construction or restoration, screening, and final model construction, and is normalized by exhaustive search. Figure~\ref{fig:selection_cost} presents the selection quality--cost comparison.\subsection{Calibration Depth at Fixed Deployment Depth}\label{app:depth_ablation}\label{sec:ablation}\label{sec:calibration-depth}The following ablations retain the original token-normalized WikiText-2 perplexity protocol with 2048-token windows.\par We vary calibration depth while fixing deployment depth for each model. Ouro-2.6B W4A8 and Ouro-1.4B W4A4 use $T_{\mathrm{cal}}\in\{1,2,4\}$ and $T_{\mathrm{deploy}}=4$; Huginn-0125 W4A8 uses $T_{\mathrm{cal}}\in\{8,16,32\}$ and $T_{\mathrm{deploy}}=32$. Within each configuration, the candidate pool and selection windows remain fixed. We reselect candidates by rollout NLL at each calibration depth, and evaluate all selected candidates at the fixed deployment depth.\begin{figure}[!htb]\centering\includegraphics[width=0.95\linewidth]{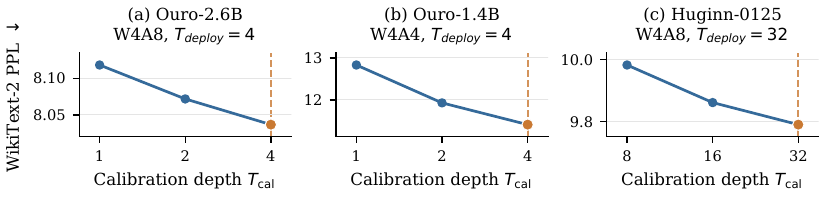}\caption{\textbf{Calibration depth at fixed deployment depth.} Token-normalized WikiText-2 perplexity with 2048-token windows for (a) Ouro-2.6B W4A8, (b) Ouro-1.4B W4A4, and (c) Huginn-0125 W4A8. Orange markers indicate matched calibration and deployment depths. Each configuration has one run; no error bars are shown. Vertical scales differ across panels.}\label{fig:calibration-depth}\end{figure}\par As shown in Figure~\ref{fig:calibration-depth}, increasing calibration depth lowers deployment perplexity in all three configurations. From the shallowest tested depth to the deployment depth, perplexity decreases from 8.1184 to 8.0367 on Ouro-2.6B, from 12.83 to 11.41 on Ouro-1.4B, and from 9.98 to 9.79 on Huginn. Recorded selection times increase from 96.7 to 296.8 seconds, 62.4 to 171.8 seconds, and 847.3 to 3164.7 seconds, respectively. Matching deployment depth therefore gives the best perplexity among the tested depths in these runs, at increased selection cost. Additional runs are needed to assess variability.

\section{Shared Transformations and Recurrent Error Propagation}\label{app:tangent} \subsection{Equivalence of Shared Transformations}\label{app:shared-transform-equivalence} \begin{proposition}[Equivalence under recurrent execution]\label{prop:shared-transform-equivalence} Consider a recurrent core $F$ with linear operators $xW_\ell^\top+b_\ell$. For each transformed operator $\ell$, let $D_\ell$ be a positive diagonal matrix and $R_\ell$ an orthogonal matrix. Replace its input and weight by \[ \widetilde x=xD_\ell^{-1}R_\ell, \qquad \widetilde W_\ell=W_\ell D_\ell R_\ell, \] while retaining its bias and all other operations. These transformations are fixed and shared across recurrent steps. Before quantization and in exact arithmetic, the resulting core $\widetilde F$ produces the same hidden-state trajectory as $F$ from the same initial state, \[ \widetilde h_t=h_t,\qquad t=0,\ldots,T, \] for any finite recurrent depth $T$. An unchanged output readout consequently produces identical final logits. \end{proposition} \noindent\textit{Proof.} For any input $x$, orthogonality of $R_\ell$ and invertibility of $D_\ell$ give \[ \begin{aligned} \widetilde x\widetilde W_\ell^\top+b_\ell &=xD_\ell^{-1}R_\ell (W_\ell D_\ell R_\ell)^\top+b_\ell\\ &=xD_\ell^{-1}R_\ell R_\ell^\top D_\ell W_\ell^\top+b_\ell\\ &=xW_\ell^\top+b_\ell. \end{aligned} \] Each transformed operator thus retains its original output coordinates and values. Since all other operations remain unchanged, $\widetilde F(h)=F(h)$ for every input $h$. Starting from $\widetilde h_0=h_0$, induction yields \[ \widetilde h_{t+1}=\widetilde F(\widetilde h_t)=F(h_t)=h_{t+1}. \] The final logits are identical under the same output readout.\hfill$\square$\par The equivalence is algebraic and requires neither contractivity nor differentiability, but does not imply bitwise agreement in finite precision. Quantization can produce representation-dependent perturbations, trajectories, and prediction losses, so pre-quantization equivalence neither guarantees equal quantized quality nor justifies selection by local reconstruction error alone. The next subsection analyzes recurrent perturbation propagation.\par
\subsection{Exact Loss Decomposition and Recurrent Transport} \label{app:recurrent-error} Fix a calibration window $x$ and its initialization, and suppress the dependence of the core maps on $x$. Let $h_{t+1}=F(h_t)$ and $\hat h_{t+1,c}=\widehat F_c(\hat h_{t,c})$, with $\hat h_{0,c}=h_0$. The candidate map $\widehat F_c$ is fixed and reused at every recurrent step. Let $\ell_x(h)$ denote the window-mean next-token NLL after a fixed output computation, so that $L_T(c,x)=\ell_x(\hat h_{T,c})$. Candidates and the reference share initialization and non-recurrent computation, which need not be fully unquantized. The reference therefore isolates recurrent-core changes and need not equal the fully floating-point baseline. All losses are assumed finite. \par \paragraph{Exact target-depth loss decomposition.} Define the local perturbation at the state actually visited by the quantized model and the loss after a reference continuation as \begin{equation}  z_{t,c}=F(\hat h_{t,c}),\qquad e_{t,c}=\widehat F_c(\hat h_{t,c})-F(\hat h_{t,c}),\qquad  V_t^{(T)}(h)=\ell_x\!\left(F^{T-t}(h)\right),  \label{eq:reference-continuation} \end{equation} where $0\leq t\leq T$ for $V_t^{(T)}$ and $F^0$ is the identity map. \begin{proposition}[Exact loss decomposition]\label{prop:exact-loss-decomposition} For any finite target depth $T$, \begin{equation}  L_T(c,x)-\ell_x(h_T)  =\sum_{t=0}^{T-1}\left[V_{t+1}^{(T)}(z_{t,c}+e_{t,c})-V_{t+1}^{(T)}(z_{t,c})\right].  \label{eq:exact-loss-decomposition} \end{equation} No differentiability, contractivity, or small-perturbation assumption is required. \end{proposition} \noindent\textit{Proof.} By definition, $V_t^{(T)}(h)=V_{t+1}^{(T)}(F(h))$ and $\hat h_{t+1,c}=z_{t,c}+e_{t,c}$. Each summand therefore equals $V_{t+1}^{(T)}(\hat h_{t+1,c})-V_t^{(T)}(\hat h_{t,c})$. Summing telescopes to \[  V_T^{(T)}(\hat h_{T,c})-V_0^{(T)}(h_0)  =\ell_x(\hat h_{T,c})-\ell_x(F^T(h_0)), \] which proves the identity.\hfill$\square$ \par The identity requires no quantizer derivatives and applies to discontinuous quantized maps. It evaluates perturbations along the quantized trajectory through reference continuations, rather than by their norms alone, and does not imply monotone error growth. \par \paragraph{Optional smooth interpretation.} For fixed $x,c,t,T$, additionally suppose $V_{t+1}^{(T)}$ is continuously differentiable on a neighborhood of the segment $\{z_{t,c}+s e_{t,c}:0\leq s\leq1\}$, with an $M_{t,c}^{(T)}$-Lipschitz gradient along that segment. The fundamental theorem of calculus gives \begin{equation} \begin{aligned}  &V_{t+1}^{(T)}(z_{t,c}+e_{t,c})-V_{t+1}^{(T)}(z_{t,c})  =\left\langle\nabla V_{t+1}^{(T)}(z_{t,c}),e_{t,c}\right\rangle+\rho_{t,c,x}^{(T)},\\  &\rho_{t,c,x}^{(T)}=\int_0^1\left\langle\nabla V_{t+1}^{(T)}(z_{t,c}+s e_{t,c})-\nabla V_{t+1}^{(T)}(z_{t,c}),e_{t,c}\right\rangle\,ds. \end{aligned} \end{equation} The Lipschitz condition implies $|\rho_{t,c,x}^{(T)}|\leq M_{t,c}^{(T)}\|e_{t,c}\|_2^2/2$. Consequently, \begin{equation} \begin{aligned}  L_T(c,x)-\ell_x(h_T)  &=\sum_{t=0}^{T-1}\left\langle\nabla V_{t+1}^{(T)}(z_{t,c}),e_{t,c}\right\rangle+R_{T,c,x},\\  |R_{T,c,x}|&\leq\frac12\sum_{t=0}^{T-1}M_{t,c}^{(T)}\|e_{t,c}\|_2^2. \end{aligned} \label{eq:smooth-loss-decomposition} \end{equation} When the chain rule applies, \[  \nabla V_{t+1}^{(T)}(z)  =\left[D(F^{T-t-1})(z)\right]^\top\nabla\ell_x\!\left(F^{T-t-1}(z)\right). \] The first-order effect depends on perturbation direction and reference continuation, with a remainder not assumed small at aggressive precision. A quantized output computation may violate the smoothness assumptions, whereas Equation~(\ref{eq:exact-loss-decomposition}) remains valid without them. \par \paragraph{Relation to local state-error transport.} Writing $\delta_{t,c}=\hat h_{t,c}-h_t$, a differentiable reference map admits the local expansion \begin{equation}  \delta_{t+1,c}=J_t\delta_{t,c}+e_{t,c}+r_{t,c},\qquad  J_t=\left.\frac{\partial F}{\partial h}\right|_{h_t},  \label{eq:transport} \end{equation} Here, $r_{t,c}$ is the Taylor remainder. Without controlling it, this recurrence alone gives no quantitative error guarantee for the quantized model. \par \paragraph{Connection to depth mismatch.} For the fixed calibration set $\mathcal D$, define \[  \mathcal L_t^{\mathrm{ref}}=\frac1{|\mathcal D|}\sum_{x\in\mathcal D}\ell_x(h_t(x)),\qquad  \Delta_t(c)=\mathcal L_t(c;\mathcal D)-\mathcal L_t^{\mathrm{ref}}. \] Averaging Equation~(\ref{eq:exact-loss-decomposition}) over windows decomposes $\Delta_t(c)$ exactly. If $d_{\tau,T}(c)=\mathcal L_\tau(c;\mathcal D)-\mathcal L_T(c;\mathcal D)$, then \[  d_{\tau,T}(c)=\mathcal L_\tau^{\mathrm{ref}}-\mathcal L_T^{\mathrm{ref}}+\Delta_\tau(c)-\Delta_T(c). \] The reference-only term is independent of $c$. Hence the depth mismatch in Equation~(\ref{eq:depth-mismatch}) satisfies \[  \beta_{\tau,T}=\max_{c\in\mathcal C}\{\Delta_\tau(c)-\Delta_T(c)\}  -\min_{c\in\mathcal C}\{\Delta_\tau(c)-\Delta_T(c)\}. \] This expresses depth mismatch through depth-dependent perturbation effects, without guaranteeing small $\beta_{\tau,T}$ or a ranking reversal when $\beta_{\tau,T}>0$. \par \subsection{A Contractive Ranking-Reversal Construction}\label{app:contractive-counterexample} We give an explicit construction showing that contractivity and smaller local perturbation norms do not guarantee agreement between shallow- and target-depth candidate rankings. \paragraph{Shared contractive dynamics.} Let $h_0=0$, $F(h)=Ah$, and \[  A=\operatorname{diag}(0.9,0.1),\qquad b_a=(0.5,0)^\top,\qquad b_b=(0,1)^\top,  \qquad \widehat F_c(h)=Ah+\eta b_c,\quad \eta>0. \] Each candidate $c\in\{a,b\}$ reuses the same map at every recurrent step. Since the additive perturbation cancels between inputs and $\|A\|_2=0.9$, \[ \|\widehat F_c(h)-\widehat F_c(h')\|_2 =\|F(h)-F(h')\|_2 =\|A(h-h')\|_2 \leq0.9\|h-h'\|_2. \] At every state, $\widehat F_c(h)-F(h)=\eta b_c$. Thus $\|e_{t,a}\|_2=0.5\eta<\eta=\|e_{t,b}\|_2$, both on and off the realized trajectories. \paragraph{Explicit trajectories and ranking reversal.} The reference trajectory is zero. Iterating the candidate recurrence gives \[  \hat h_{T,c}=\eta\sum_{j=0}^{T-1}A^j b_c. \] For binary logits $(0,u(h))$, where $u(h)=h_1+h_2$, and the first class as target, the NLL is $\ell(h)=\log(1+\exp(u(h)))$. It is strictly increasing in $u(h)$, and \begin{equation}  u(\hat h_{T,a})=\frac{\eta}{2}\sum_{j=0}^{T-1}0.9^j,\qquad  u(\hat h_{T,b})=\eta\sum_{j=0}^{T-1}0.1^j.  \label{eq:contractive-trajectories} \end{equation} In particular, \[  \begin{array}{c|cc|c}  \text{Depth} & u(\hat h_{T,a})/\eta & u(\hat h_{T,b})/\eta & \text{Lower NLL}\\  \hline  1 & 0.5 & 1 & a\\  4 & 1.7195 & 1.111 & b  \end{array} \] Writing $L_t(c)=\ell(\hat h_{t,c})$, we obtain $L_1(a)<L_1(b)$ but $L_4(a)>L_4(b)$. Both inequalities are strict for every $\eta>0$, including arbitrarily small perturbations. Candidate $a$ perturbs the more slowly decaying coordinate, so its smaller local errors can have a larger effect after repeated execution. \paragraph{Selection failure without window-estimation error.} For $\mathcal C=\{a,b\}$, a singleton calibration set, and target depth $T=4$, exact depth-1 screening that retains one candidate selects $a$ and eliminates the target-depth optimum $b$, despite zero window-estimation error. Exact target-depth evaluation of the survivor cannot recover $b$, yielding \begin{equation}  \operatorname{Reg}_4(a)=\log(1+\exp(1.7195\eta))-\log(1+\exp(1.111\eta))>0.  \label{eq:contractive-selection-regret} \end{equation} For $\eta=1$, the target-depth NLLs are approximately $1.884299$ and $1.395600$, giving regret $0.488699$. These are values of the analytical construction, not language-model measurements. \paragraph{Scope.} This abstract construction rules out guarantees based solely on contractivity and local-error norms. It does not specify quantization settings realizing these affine maps or assert contractivity of the evaluated language-model cores. Nor does it imply that shallow screening always fails or that full-depth screening is uniformly superior at equal compute. \par  \subsection{Proof of the Selection-Regret Bound}\label{app:selection-regret} To prove Equation~(\ref{eq:selection-regret-bound}), fix a finite nonempty candidate pool $\mathcal C$ and calibration set $\mathcal D$, with fixed representations, calibration bases, and window initializations. Write $\mathcal L_t(c)=\mathcal L_t(c;\mathcal D)$ and assume finite losses. \paragraph{Depth mismatch.} Set $d_{\tau,T}(c)=\mathcal L_\tau(c)-\mathcal L_T(c)$. The pairwise mismatch in Equation~(\ref{eq:depth-mismatch}) satisfies \begin{equation}  \beta_{\tau,T}=\max_{c,c'\in\mathcal C}|d_{\tau,T}(c)-d_{\tau,T}(c')|  =\max_{c\in\mathcal C}d_{\tau,T}(c)-\min_{c\in\mathcal C}d_{\tau,T}(c).  \label{eq:depth-mismatch-range} \end{equation} The maximum pairwise absolute difference equals the range, so candidate-independent loss offsets leave $\beta_{\tau,T}$ unchanged, and $\beta_{T,T}=0$. \paragraph{Screening protocol and assumptions.} Start with $\mathcal C_1=\mathcal C$. At each stage $1\leq r<R$, score the surviving candidates at depth $\tau_r\in\{1,\ldots,T\}$ on a nonempty subset $\mathcal D_r\subseteq\mathcal D$: \[  \widehat L_r(c)=\frac1{|\mathcal D_r|}\sum_{x\in\mathcal D_r}L_{\tau_r}(c,x),\qquad c\in\mathcal C_r. \] Retain the $K_{r+1}$ lowest-scoring candidates, where $1\leq K_{r+1}\leq|\mathcal C_r|$, with fixed tie breaking. Candidate sets are nested and eliminated candidates are never reintroduced. Assume that each screening stage satisfies \begin{equation}  \max_{c\in\mathcal C_r}|\widehat L_r(c)-\mathcal L_{\tau_r}(c)|\leq\varepsilon_r.  \label{eq:stage-window-error} \end{equation} The final stage evaluates all survivors on the complete calibration set at the target depth and returns $\hat c_T\in\arg\min_{c\in\mathcal C_R}\mathcal L_T(c)$. For $R\geq2$, we will show \begin{equation}  0\leq\operatorname{Reg}_T(\hat c_T)  :=\mathcal L_T(\hat c_T)-\min_{c\in\mathcal C}\mathcal L_T(c)  \leq\max_{1\leq r<R}(\beta_{\tau_r,T}+2\varepsilon_r).  \label{eq:appendix-selection-regret} \end{equation} \paragraph{Proof.} Let $L_T^\star=\min_{c\in\mathcal C}\mathcal L_T(c)$ and $\mathcal C_T^\star=\{c\in\mathcal C:\mathcal L_T(c)=L_T^\star\}$. If $\mathcal C_R\cap\mathcal C_T^\star\neq\varnothing$, exact final selection returns a global minimizer and the regret is zero. No uniqueness assumption is needed. \par Otherwise, define the stage at which the last global minimizers are discarded: \[  r^\dagger=\min\{r:\mathcal C_{r+1}\cap\mathcal C_T^\star=\varnothing\}. \] Since $\mathcal C_1=\mathcal C$ and the sets are nested, $1\leq r^\dagger<R$, and some $c^\star\in\mathcal C_{r^\dagger}\cap\mathcal C_T^\star$ exists. Every retained candidate $d\in\mathcal C_{r^\dagger+1}$ satisfies $\widehat L_{r^\dagger}(d)\leq\widehat L_{r^\dagger}(c^\star)$, including under tied scores. Applying Equation~(\ref{eq:stage-window-error}) to both candidates yields \[  \mathcal L_{\tau_{r^\dagger}}(d)-\mathcal L_{\tau_{r^\dagger}}(c^\star)\leq2\varepsilon_{r^\dagger}. \] By the definition of pairwise depth mismatch, \[  \begin{aligned} \mathcal L_T(d)-L_T^\star &\leq\mathcal L_{\tau_{r^\dagger}}(d)-\mathcal L_{\tau_{r^\dagger}}(c^\star)+\beta_{\tau_{r^\dagger},T}\\ &\leq2\varepsilon_{r^\dagger}+\beta_{\tau_{r^\dagger},T}. \end{aligned} \] This holds for every candidate retained at that stage. All later candidate sets are subsets of $\mathcal C_{r^\dagger+1}$, so it also holds for $\hat c_T$. Thus \begin{equation}  \operatorname{Reg}_T(\hat c_T)\leq\beta_{\tau_{r^\dagger},T}+2\varepsilon_{r^\dagger}.  \label{eq:last-optimum-regret} \end{equation} Taking the maximum over screening stages proves the upper bound. Nonnegativity follows from $\hat c_T\in\mathcal C$.\hfill$\square$ \paragraph{Full-depth consequence.} For \method{}, $\tau_r=T$ at every screening stage, so $\beta_{\tau_r,T}=0$ and \[  \operatorname{Reg}_T(\hat c_T)\leq2\max_{1\leq r<R}\varepsilon_r, \] recovering Equation~(\ref{eq:full-depth-regret}). If all global minimizers are eliminated, the refined bound is $2\varepsilon_{r^\dagger}$. The stage-maximum bound also implies the weaker sum bound because all terms are nonnegative. Accumulation across stages is unnecessary for nested top-$K$ screening: later elimination cannot escape the loss bound already satisfied by every retained candidate. If $R=1$, there is no screening and exact full-set selection has zero regret. \paragraph{Assumptions and limits.} This deterministic guarantee is conditional on the fixed-pool protocol and stagewise error bounds above. It requires neither a unique optimum nor a positive gap, and assumes no window independence, unbiasedness, bounded NLL, or distributional model for screening errors. The schedule does not certify $\varepsilon_r$, and the proof supplies no concentration or sample-complexity guarantee. The result concerns the fixed pool and calibration set, not all quantizers or held-out generalization. Nonzero depth mismatch neither lower-bounds regret nor necessarily reverses rankings, and removing it does not guarantee superiority at equal compute. The next proposition gives sufficient conditions for exact optimum recovery. \par  \subsection{Optimal-Candidate Preservation}\label{app:optimal-candidate-preservation} Complementing the regret bound in Appendix~\ref{app:selection-regret}, we give a sufficient condition for exact recovery of a unique optimum under full-depth screening. Retain its fixed candidate pool, calibration set, and representations, and write $L(c)=\mathcal L_T(c;\mathcal D)$. \begin{proposition}[Optimal-candidate preservation]\label{prop:optimal-candidate-preservation} Assume that $L$ has a unique minimizer $c^\star$ over the finite nonempty pool $\mathcal C$. Start with $\mathcal C_1=\mathcal C$. At each screening stage $1\leq r<R$, score candidates at the target depth $T$ using estimates $\widehat L_r(c)$ and retain the $K_{r+1}$ lowest-scoring candidates, with fixed tie breaking and $1\leq K_{r+1}\leq|\mathcal C_r|$. Eliminated candidates are not reintroduced. Suppose \begin{equation}  \max_{c\in\mathcal C_r}|\widehat L_r(c)-L(c)|\leq\varepsilon_r.  \label{eq:preservation-estimation-error} \end{equation} Let $L_{r,(j)}$ be the $j$-th smallest full-set loss among candidates in $\mathcal C_r$, counted with multiplicity. At every stage that eliminates candidates, assume \begin{equation}  L_{r,(K_{r+1}+1)}-L(c^\star)>2\varepsilon_r.  \label{eq:preservation-gap} \end{equation} Then $c^\star$ survives all screening stages. Exact final evaluation of the survivors on $\mathcal D$ at depth $T$ returns $c^\star$, and the selection regret is zero. \end{proposition} \noindent\textit{Proof.} Initially, $c^\star\in\mathcal C_1$. Suppose $c^\star\in\mathcal C_r$. If $K_{r+1}=|\mathcal C_r|$, no candidate is discarded and the claim is immediate. Otherwise, write $K=K_{r+1}$. For every candidate satisfying $L(c)\geq L_{r,(K+1)}$, Equations~(\ref{eq:preservation-estimation-error}) and~(\ref{eq:preservation-gap}) imply \[  \widehat L_r(c)\geq L(c)-\varepsilon_r  >L(c^\star)+\varepsilon_r\geq\widehat L_r(c^\star). \] Consequently, any candidate whose estimated loss is no greater than that of $c^\star$ must have full-set loss strictly below $L_{r,(K+1)}$. There are at most $K$ such candidates, including $c^\star$ itself. Thus at most $K-1$ other candidates can precede $c^\star$, even with tied estimated scores, and $c^\star$ is retained. Induction proves survival through every stage. Since the unique global minimizer belongs to $\mathcal C_R$, exact final evaluation returns it.\hfill$\square$ \paragraph{Interpretation and relation to the regret bound.} The gap compares the optimum with the first loss outside the retention budget, not necessarily the second-best candidate. For a fixed candidate set, retaining more candidates weakens the separation requirement. More windows may improve estimation but need not monotonically reduce realized error. The strict gap guarantees survival under any fixed tie-breaking rule. \par Unlike the regret bound in Appendix~\ref{app:selection-regret}, exact recovery here requires uniqueness and separation. The regret bound remains applicable without either condition, including after all optima are discarded. Failure of Equation~(\ref{eq:preservation-gap}) does not imply elimination of the optimum, since the condition is sufficient, not necessary. \paragraph{Scope.} The guarantee is conditional on the stated uniform estimation-error bounds and candidate separation, which the fixed schedule does not certify. It does not establish recovery under arbitrary shallow-depth scores, optimality beyond the fixed candidate pool, held-out generalization, or superiority at equal compute. Empirical oracle recovery and selection regret characterize complementary aspects of approximation quality but do not verify these stagewise conditions. \par \section{H20 Decoding Throughput}\label{app:h20-throughput}
\textbf{Evaluation protocol.}
We benchmark Ouro-1.4B at $T=4$ on a NVIDIA H20 GPU, using batch size/input length pairs of 1/128 and 4/2048. Each run generates 128 tokens without EOS stopping, with the first token produced during prefill and the remaining 127 through cached decoding. Aggregate batch throughput is $127B/t_{\mathrm{decode}}$, measured with CUDA-synchronized wall-clock timing over the complete decoding loop, including Python scheduling but excluding prefill and initialization. All methods use fresh BF16 KV caches with 96 recurrent slots, Flash SDPA, and eager execution without CUDA Graphs or profiling. We report the arithmetic mean of ten throughput measurements, collected in two blocks of three warm-up runs followed by five fixed-input measurements.

\textbf{Offline and hybrid configurations.}
Loopy (offline) combines Q/K/V- and gate/up-shared input scaling absorbed into RMSNorm and linear weights, head-wise V--O Hadamard rotation with block size 128 absorbed into paired weights, and up--down scaling that does not cross the SiLU gate. These transforms are shared across recurrent calls and require no explicit online transformation, while activation quantization remains dynamic. This constrained configuration differs from the online Loopy checkpoint. The two hybrid configurations add down-input or down- and Q/K/V-input rotations using seed 42 and block size 128 with matching weights. All three configurations reuse compatible input quantization results across Q/K/V and gate/up without fusing their matrix multiplications.

\textbf{Kernel differences and diagnostic entries.}
Loopy uses group-128 Triton packed-W4 kernels, whereas the adapted QuaRot and SpinQuant kernels use per-output-channel weight scales, with additional asymmetric activation zero-point correction for SpinQuant. Attention and timing scope are shared, so the results compare implementations rather than isolate transformation overhead. Entries marked $*$ denote W4A4 integer paths whose NLL differences from BF16 evaluation exceed the predefined tolerance; their throughput cannot be paired directly with BF16-path quality results. W4A8 passes this check on four calibration windows, but full quality evaluation was not repeated on the timed execution paths.

\begin{table}[t]
\centering\small
\caption{Complete H20 decoding throughput (tokens/s, $\uparrow$). The last two rows add online rotations to Loopy (offline).}
\label{tab:h20-full}
\setlength{\tabcolsep}{6pt}
\begin{tabular}{lrrrr}
\toprule
& \multicolumn{2}{c}{W4A4} & \multicolumn{2}{c}{W4A8} \\
\cmidrule(lr){2-3}\cmidrule(lr){4-5}
Method & B1/128 & B4/2048 & B1/128 & B4/2048 \\
\midrule
QuaRot & $9.08^{*}$ & $35.14^{*}$ & 8.80 & 34.75 \\
SpinQuant & $9.16^{*}$ & $35.85^{*}$ & 9.14 & 35.78 \\
\rowcolor[gray]{0.93}
Loopy (online) & 10.01 & 40.93 & 10.22 & 41.17 \\
\rowcolor[gray]{0.93}
Loopy (offline) & \textbf{11.21} & 45.45 & \textbf{11.37} & \textbf{45.83} \\
\rowcolor[gray]{0.93}
+ down rotation & 11.17 & \textbf{45.56} & \textbf{11.37} & 45.60 \\
\rowcolor[gray]{0.93}
+ down/QKV rotation & 11.15 & 45.26 & 11.25 & 45.66 \\
\bottomrule
\end{tabular}
\end{table}

\section{Evidence and Reproduction}\label{app:provenance}
The reported results reflect the current experimental pipeline. The main Loopy quantization results use configurations selected by progressive full-depth evaluation, with full-precision results provided as references. Supporting materials include evaluation reports, JSON records for diagnostic figures, and complete-test records for the rebuilt packed checkpoint. 

Figure inputs are stored as paper-local snapshots with source hashes. The accompanying records document the configurations, evaluation protocols, and execution paths associated with the reported results. Hashes identify data and checkpoint versions; numerical agreement and prediction quality are assessed through the corresponding experiments. The supplementary evidence ledger records the scope of these checks.

\section{Activation Distributions Across Transformer Blocks}\label{app:activation-distributions}
To examine whether the activation behavior illustrated in Figure~\ref{fig:rotation-benefit}(a) extends beyond the displayed block, we visualize the attention and FFN inputs of all 24 physical Transformer blocks in the recurrent core of Ouro-1.4B at recurrent step 4. All views use the same 2048-token training window, identity scaling, and block-Hadamard rotation with block size 128 and seed 42. Across all 48 block--input pairs, rotation reduces both the maximum absolute activation and the mean per-token maximum-to-RMS ratio. The figures retain all 2048 channels and display maximum absolute values within consecutive 16-token bins; vertical axes are scaled independently, while color ranges are shared within each before--after pair. Blocks are shown in order without omission, extending the block-22 example in the main text; these observations concern the fixed window and recurrent step evaluated here.

\begin{figure}[!htbp]
\centering
\includegraphics[width=\linewidth]{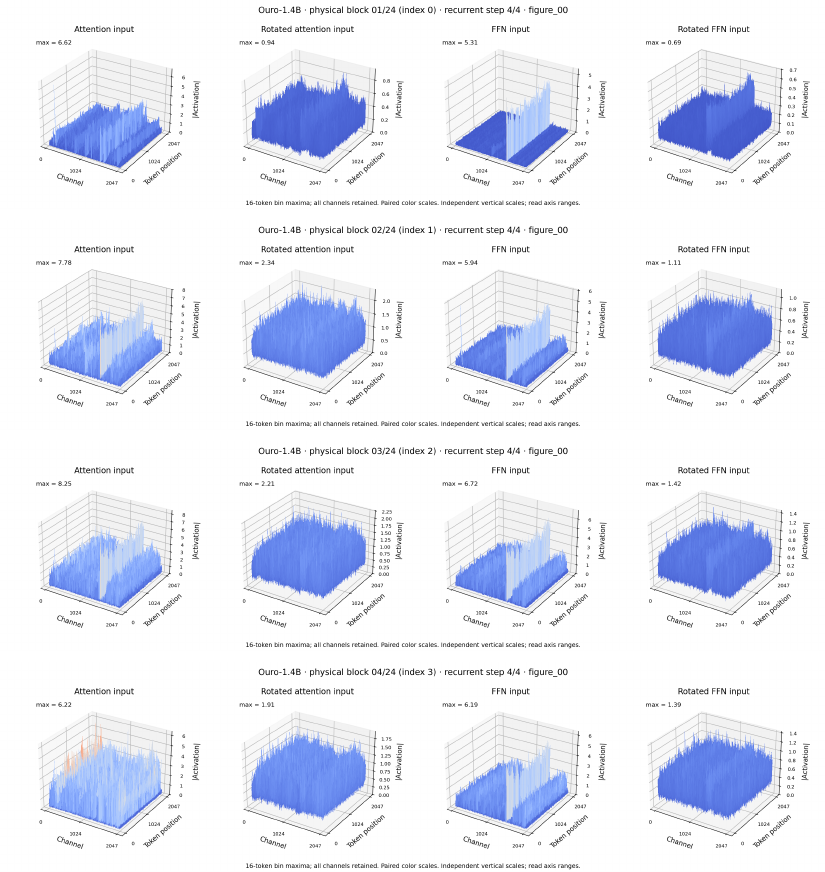}
\caption{\textbf{Activation distributions in blocks 1--4 of Ouro-1.4B at recurrent step 4.} Each row shows attention input, rotated attention input, FFN input, and rotated FFN input, from left to right.}
\label{fig:activation-blocks-01-04}
\end{figure}
\clearpage

\begin{figure}[!htbp]
\centering
\includegraphics[width=\linewidth]{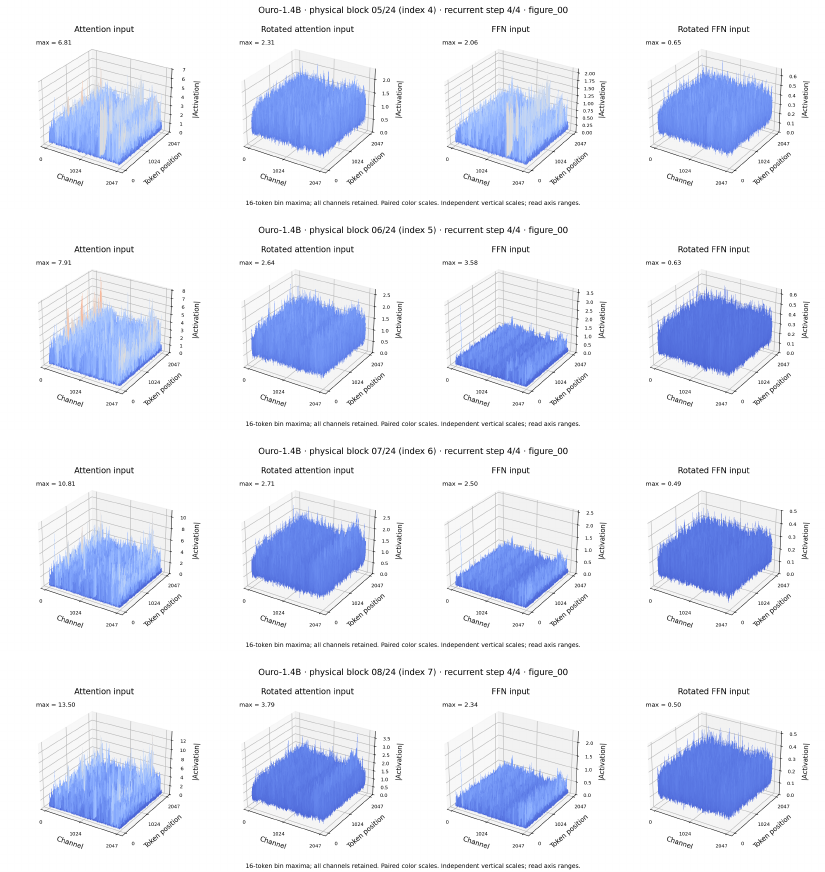}
\caption{\textbf{Activation distributions in blocks 5--8 of Ouro-1.4B at recurrent step 4.} Each row shows attention input, rotated attention input, FFN input, and rotated FFN input, from left to right.}
\label{fig:activation-blocks-05-08}
\end{figure}
\clearpage

\begin{figure}[!htbp]
\centering
\includegraphics[width=\linewidth]{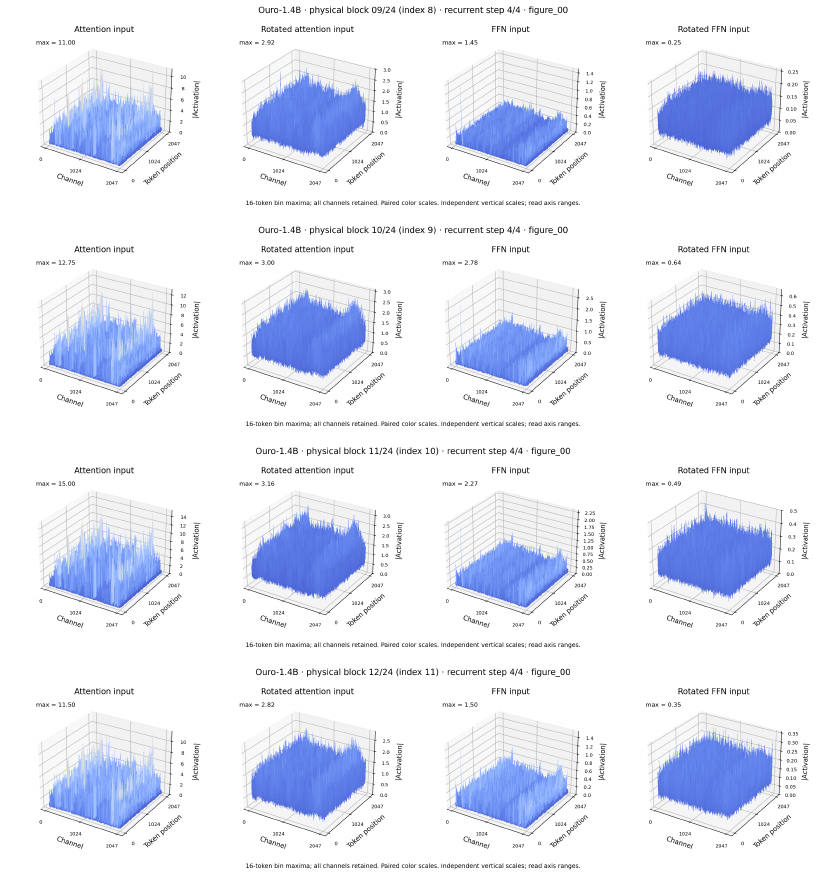}
\caption{\textbf{Activation distributions in blocks 9--12 of Ouro-1.4B at recurrent step 4.} Each row shows attention input, rotated attention input, FFN input, and rotated FFN input, from left to right.}
\label{fig:activation-blocks-09-12}
\end{figure}
\clearpage

\begin{figure}[!htbp]
\centering
\includegraphics[width=\linewidth]{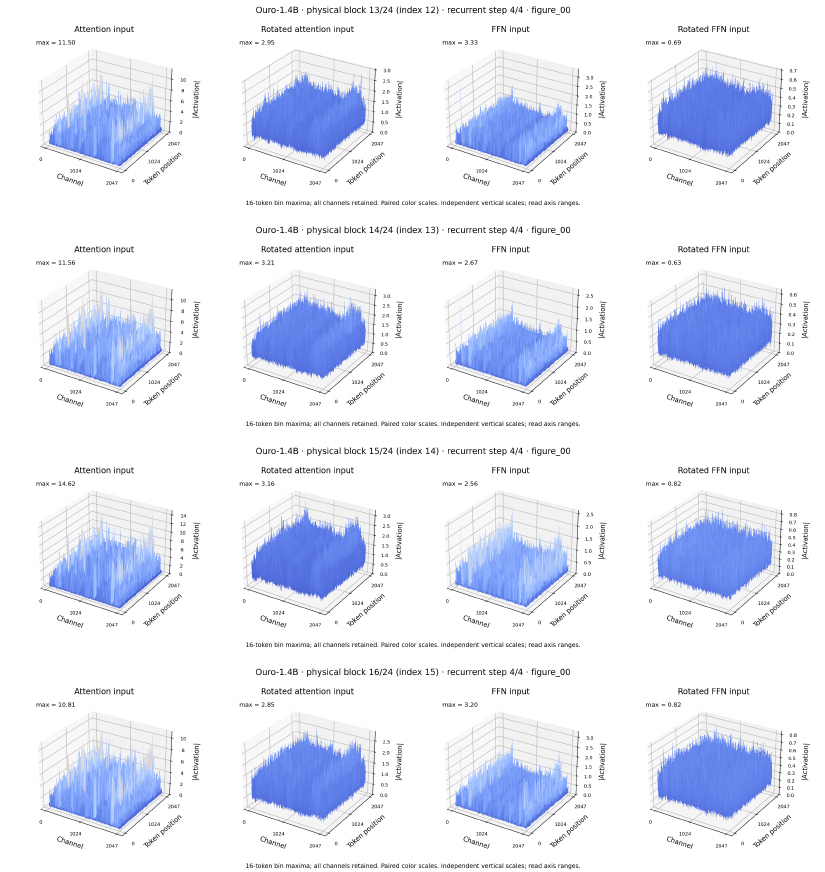}
\caption{\textbf{Activation distributions in blocks 13--16 of Ouro-1.4B at recurrent step 4.} Each row shows attention input, rotated attention input, FFN input, and rotated FFN input, from left to right.}
\label{fig:activation-blocks-13-16}
\end{figure}
\clearpage

\begin{figure}[!htbp]
\centering
\includegraphics[width=\linewidth]{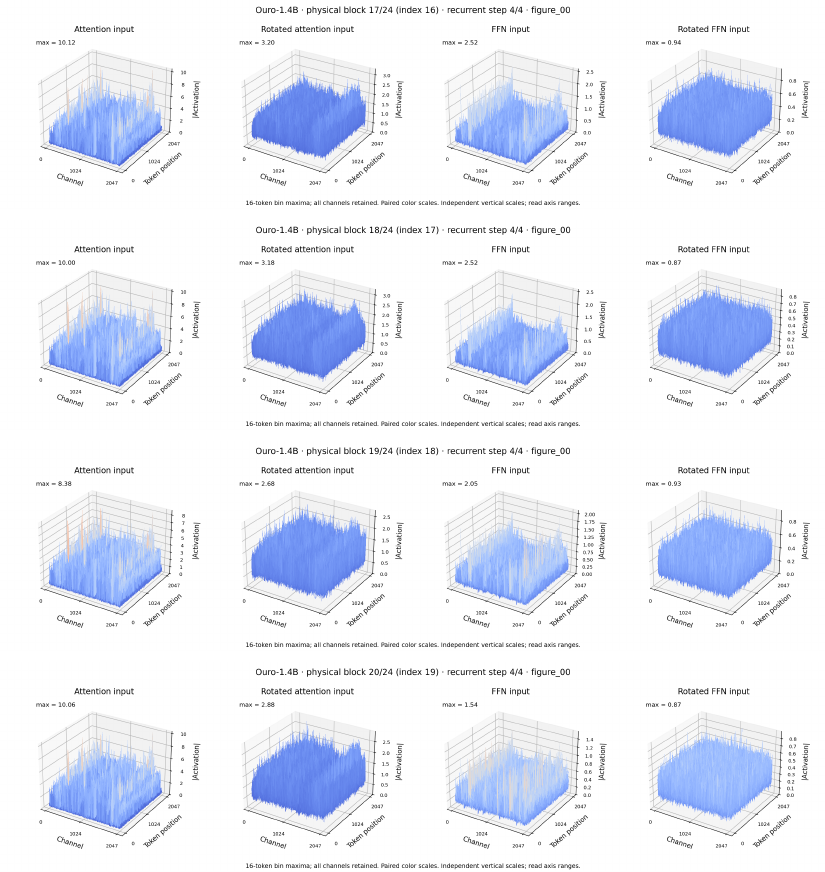}
\caption{\textbf{Activation distributions in blocks 17--20 of Ouro-1.4B at recurrent step 4.} Each row shows attention input, rotated attention input, FFN input, and rotated FFN input, from left to right.}
\label{fig:activation-blocks-17-20}
\end{figure}
\clearpage

\begin{figure}[!htbp]
\centering
\includegraphics[width=\linewidth]{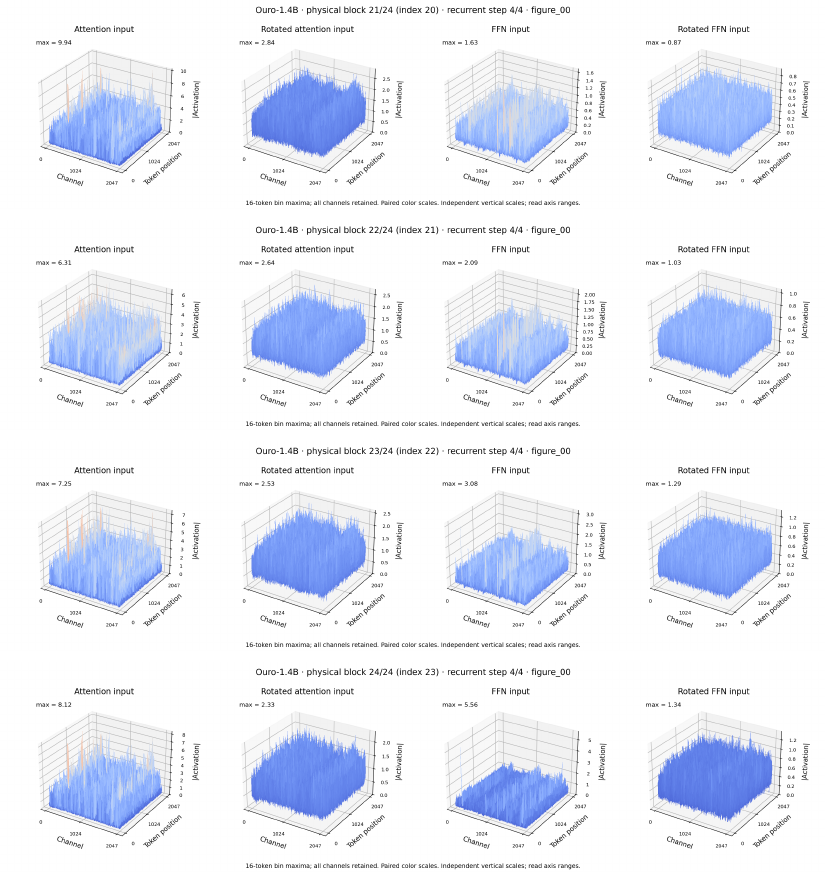}
\caption{\textbf{Activation distributions in blocks 21--24 of Ouro-1.4B at recurrent step 4.} Each row shows attention input, rotated attention input, FFN input, and rotated FFN input, from left to right.}
\label{fig:activation-blocks-21-24}
\end{figure}
\clearpage

\section{Activation Distributions Across Recurrent Steps}\label{app:activation-steps}
\subsection{Activation Profiles}
To complement the comparisons across physical blocks in Appendix~\ref{app:activation-distributions}, we visualize the input activations to the query projection in physical block 12 of Ouro-1.4B across four recurrent calls, using the same input window and the same model state as in the preceding activation analysis. The panels exhibit similar overall structure but differ in their local peaks and maximum magnitudes, illustrating that repeated use of a shared module does not imply identical activation patterns. These views provide a qualitative illustration of activation evolution during recurrent execution.

\begin{figure}[!htbp]
\centering
\includegraphics[width=0.9\linewidth]{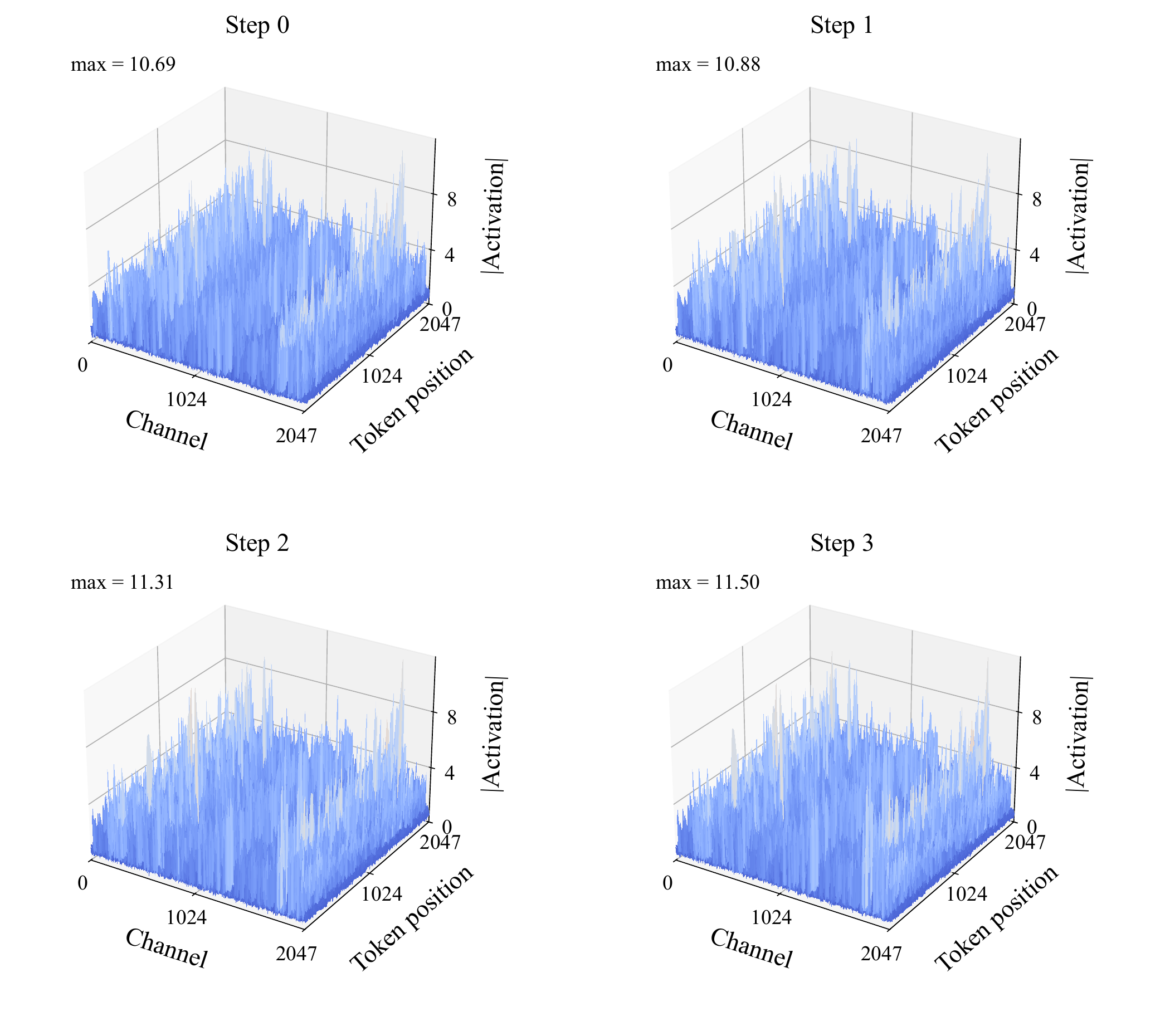}
\caption{\textbf{Activation distributions across recurrent steps.} Input activations to the query projection in physical block 12 of Ouro-1.4B for the same input window. Panels labeled Step 0--3 correspond to the first through fourth recurrent calls. The overall activation structure remains similar, while local peaks and maximum magnitudes vary across calls despite shared weights.}
\label{fig:activation-steps}
\end{figure}
\clearpage

\begin{figure}[!t]
\centering
\includegraphics[width=\linewidth]{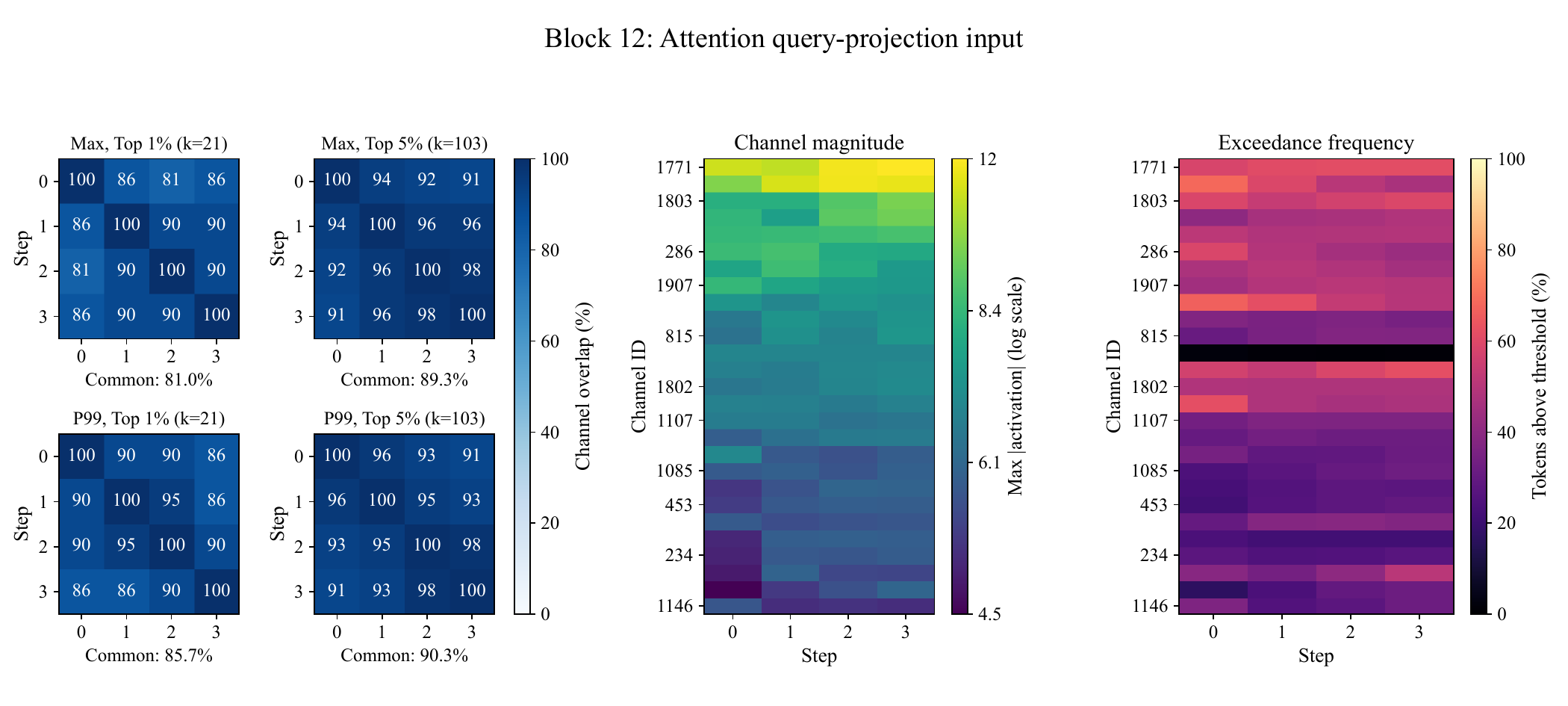}
\caption{\textbf{Outlier-channel stability at the block-12 query-projection input.} Left to right, stepwise Top-1\%/5\% channel overlap, maximum absolute activation for fixed channels (logarithmic color scale), and token exceedance frequency. The frequency threshold is 1.84375; the four-step common Max Top-1\% fraction is 81.0\%.}
\label{fig:channel-stability-query}
\medskip
\includegraphics[width=\linewidth]{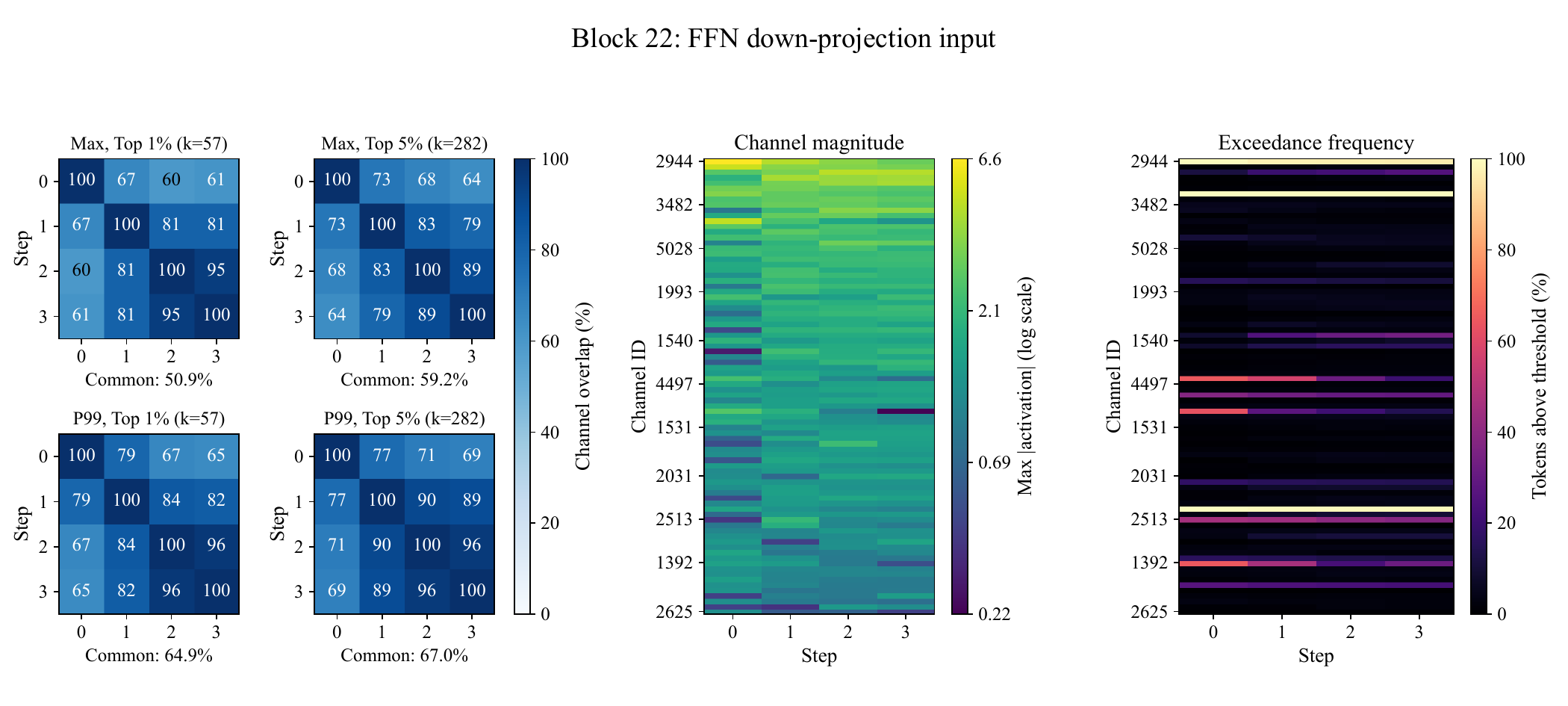}
\caption{\textbf{Outlier-channel stability at the block-22 down-projection input.} Panels follow the same layout as Figure~\ref{fig:channel-stability-query}. The frequency threshold is 0.11963. The four-step common Max Top-1\% fraction is 50.9\%, indicating greater channel-set variation than in the block-12 query case.}
\label{fig:channel-stability-down}
\end{figure}

\subsection{Outlier-Channel Stability}\label{app:channel-stability}
We examine inputs to the attention query projection in physical block 12 and the FFN down projection in physical block 22 of ByteDance Ouro-1.4B. Both cases use the same 2048-token window and an unquantized, unrotated BF16 forward pass. Steps 0--3 denote four recurrent calls. Channels are ranked separately at each step by their maximum (Max) or 99th-percentile (P99) absolute activation over tokens. Common denotes the four-step intersection size divided by the per-step Top-$k$ set size. The magnitude and frequency panels fix the union of the four Max Top-1\% sets, ordered by mean Max across steps (27 channels for block 12 and 84 for block 22). Frequency measures the fraction of tokens exceeding a case-specific threshold, fixed across steps at the pooled absolute-activation P99. The common Max Top-1\% fraction is 81.0\% in block 12 and 50.9\% in block 22. These cases show partial persistence of high-magnitude channels alongside changes in channel membership, magnitude, and exceedance frequency.

\clearpage

\section{Configuration Rankings Across Recurrent Depths}\label{app:configuration-rankings}
Figure~\ref{fig:accuracy-rank-depth} complements the absolute accuracies in Figure~\ref{fig:accuracy-depth} with within-depth rankings for the same 20 configurations of Ouro-1.4B W4A4 on WinoGrande. The fixed column order makes changes in relative performance visible as recurrent depth increases. The configuration $(p,s)=(0.8,2026)$ ranks first at depths 1, 3, and 4, while the marked winner at depth 2 is $(0.9,2026)$. Thus, the ordering varies across depths even though the shallowest and final-depth winners coincide in this experiment. Rank changes should be interpreted alongside absolute accuracy differences rather than as evidence of their statistical significance.

\begin{figure}[!htbp]
\centering
\includegraphics[width=\linewidth]{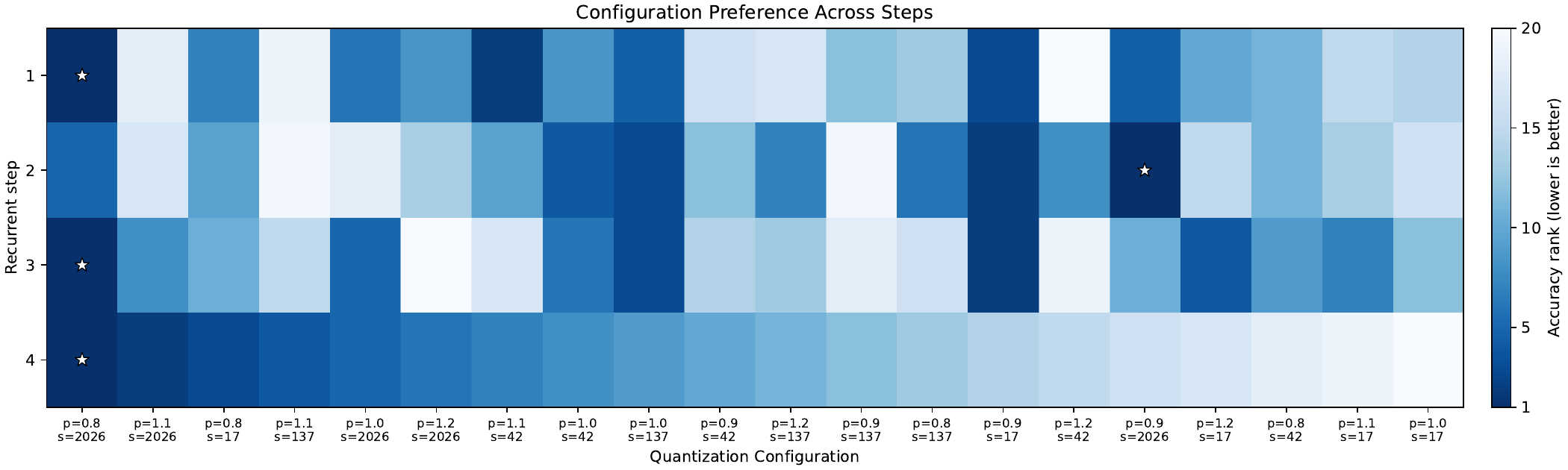}
\caption{\textbf{Within-depth configuration rankings.} WinoGrande accuracy ranks for Ouro-1.4B W4A4, using the same candidate pool and column order as Figure~\ref{fig:accuracy-depth}. Darker cells indicate better ranks (1 is best); stars mark the displayed top-ranked configuration at each depth.}
\label{fig:accuracy-rank-depth}
\end{figure}

\end{document}